\documentclass[sigconf]{acmart}
\AtBeginDocument{%
  \providecommand\BibTeX{{%
    \normalfont B\kern-0.5em{\scshape i\kern-0.25em b}\kern-0.8em\TeX}}}

\copyrightyear{2024}
\acmYear{2024}
\setcopyright{rightsretained}
\acmConference[WWW '24]{Proceedings of the ACM Web Conference 2024}{May 13--17, 2024}{Singapore, Singapore}
\acmBooktitle{Proceedings of the ACM Web Conference 2024 (WWW '24), May 13--17, 2024, Singapore, Singapore}\acmDOI{10.1145/3589334.3645648}
\acmISBN{979-8-4007-0171-9/24/05}

\makeatletter
\gdef\@copyrightpermission{
  \begin{minipage}{0.3\columnwidth}
   \href{https://creativecommons.org/licenses/by/4.0/}{\includegraphics[width=0.90\textwidth]{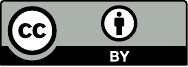}}
  \end{minipage}\hfill
  \begin{minipage}{0.7\columnwidth}
   \href{https://creativecommons.org/licenses/by/4.0/}{This work is licensed under a Creative Commons Attribution International 4.0 License.}
  \end{minipage}
  \vspace{5pt}
}
\makeatother

\usepackage{amsthm}
\usepackage{mathtools}
\usepackage{multirow}
\usepackage{mathpartir}
\usepackage{bm}
\usepackage[noabbrev, capitalize]{cleveref}
\usepackage{subcaption}
\usepackage{xspace}
\usepackage{balance}
\newcommand{\todo}[1]{}
\renewcommand{\todo}[1]{{\color{orange} TODO: {#1}}}
\newcommand{\todom}[1]{}
\renewcommand{\todom}[1]{{\color{orange} \textbf{Mathias}: {#1}}}
\newcommand{\question}[1]{}
\renewcommand{\question}[1]{{\color{red} {#1}}}

\newcommand{\boxsqel}{\textsf{Box$^\textsf{2}$EL}}
\newcommand{\emelpp}{EmEL$^{++}$}
\newcommand{\el}{$\mathcal{EL^{++}}$\xspace}
\newcommand{\alc}{$\mathcal{ALC}$}
\newcommand{\ont}{\mathcal{O}}
\newcommand{\bc}{\mathsf{BC}_\mathcal{T}}
\newcommand{\boxx}{\textnormal{Box}}
\newcommand{\bump}{\textnormal{Bump}}
\newcommand{\head}{\textnormal{Head}}
\newcommand{\tail}{\textnormal{Tail}}

\newcommand{\tsf}{\textsf}
\newcommand{\asf}{\mathsf}
\newcommand{\loss}{\mathcal L}
\newcommand{\test}{\mathcal D_\text{test}}

\DeclareMathOperator{\rk}{rk}

\DeclarePairedDelimiter\abs{\lvert}{\rvert}%
\DeclarePairedDelimiter\norm{\lVert}{\rVert}%
\makeatletter
\let\oldabs\abs
\def\abs{\@ifstar{\oldabs}{\oldabs*}}
\let\oldnorm\norm
\def\norm{\@ifstar{\oldnorm}{\oldnorm*}}
\makeatother

\newtheorem{theorem}{Theorem}
\newtheorem{lemma}{Lemma}

\theoremstyle{definition}
\newtheorem{example}{Example}
\newenvironment{proofsketch}{%
    \renewcommand{\proofname}{Proof sketch}\proof}{\endproof}

\begin{document}

\title{Dual Box Embeddings for the Description Logic \texorpdfstring{\el}{EL++}}

\author{Mathias Jackermeier}
\affiliation{%
  \institution{University of Oxford}
  \city{Oxford}
  \country{UK}
}
\email{mathias.jackermeier@cs.ox.ac.uk}

\author{Jiaoyan Chen}
\affiliation{%
  \institution{The University of Manchester}
  \city{Manchester}
  \country{UK}
}
\email{jiaoyan.chen@manchester.ac.uk}

\author{Ian Horrocks}
\affiliation{%
  \institution{University of Oxford}
  \city{Oxford}
  \country{UK}
}
\email{ian.horrocks@cs.ox.ac.uk}


\begin{abstract}
OWL ontologies, whose formal semantics are rooted in Description Logic (DL), have been widely used for knowledge representation. Similar to Knowledge Graphs (KGs), ontologies are often incomplete, and maintaining and constructing them has proved challenging. While classical deductive reasoning algorithms use the precise formal semantics of an ontology to predict missing facts, recent years have witnessed growing interest in \textit{inductive} reasoning techniques that can derive \textit{probable} facts from an ontology. Similar to KGs, a promising approach is to learn ontology embeddings in a latent vector space, while additionally ensuring they adhere to the semantics of the underlying DL\@. While a variety of approaches have been proposed, current ontology embedding methods suffer from several shortcomings, especially that they all fail to faithfully model one-to-many, many-to-one, and many-to-many relations and role inclusion axioms.
 To address this problem and improve ontology completion performance, we propose a novel ontology embedding method named \boxsqel\ for the DL \el, which represents both concepts and roles as boxes (i.e., axis-aligned hyperrectangles), and models inter-concept relationships using a bumping mechanism.
 We theoretically prove the soundness of \boxsqel\ and conduct an extensive experimental evaluation, achieving state-of-the-art results across a variety of datasets on the tasks of subsumption prediction, role assertion prediction, and approximating deductive reasoning.\footnote{Code and data are available at \url{https://github.com/KRR-Oxford/BoxSquaredEL}.
 }
\end{abstract}



\maketitle

\section{Introduction}
Ontologies are a widely used formalism to represent general knowledge about a domain \cite{staab2010handbook}. They are usually specified in the Web Ontology Language (OWL)~\citep{grau2008OWL} standard developed by W3C\footnote{\url{https://www.w3.org/OWL/}}, and have been widely adopted in many domains such as the Semantic Web~\citep{heder2014semantic}, healthcare~\citep{schulz2009SNOMED}, bioinformatics~\citep{hoehndorf2011common}, and geoinformatics~\citep{zhai2010geospatial}. OWL allows for the expression of a variety of statements, ranging from simple relational facts to specifying concept hierarchies and complex logical relationships, and is underpinned by Description Logic (DL) \citep{baader2005description} to define its formal semantics.

Although many real-world OWL ontologies have been developed and used with great success, such as the Gene Ontology GO \cite{ashburner2000gene} and the food ontology FoodOn \cite{dooley2018foodon}, both maintaining these existing OWL ontologies and creating new ontologies has proved challenging and relies mostly on manual labor carried out by experts. Common ontology curation tasks include completing missing subsumptions between concepts (or membership relations between individuals and concepts) and identifying missing logical restrictions between concepts.
Symbolic logical reasoning algorithms such as HermiT \cite{glimm2014HermiT} and ELK \cite{kazakov2014Incredible} help address this problem by deductively inferring implicit knowledge from the precise semantics of an ontology, but this classical reasoning is often too rigid for real-world OWL ontologies\,---\, especially in the presence of incomplete or noisy data\,---\, and cannot derive knowledge that is only \textit{probable} from the given ontology.

At the same time, there has been growing interest in representation learning-based methods for completing Knowledge Graphs (KGs) \cite{hogan2021knowledge}, i.e.\ relational facts in the form of RDF\footnote{Resource Description Framework. \url{https://www.w3.org/RDF/}.} triples <Subject, Predicate, Object>. Most of these approaches first learn structure-preserving \textit{embeddings} of the entities and relations (predicates) of a KG in a latent vector space and then use them to score the likelihood of novel facts \citep{wang2017knowledge}.
For example, the classic method TransE \cite{bordes2013translating} maps entities and relations to vectors such that translating the subject embedding by the relation embedding approximately yields the object embedding.

Similar embedding-based techniques as for KGs have been developed for inductive reasoning in ontologies, which promises to complement classical deductive reasoning for ontology curation tasks.
Some approaches such as OPA2Vec \cite{smaili2019opa2vec} and OWL2Vec* \cite{chen2021owl2vec} rely on exploiting textual meta information (e.g., concept labels and comments) to model similarities between entities, but do not retain the semantics defined by the underlying DL\@. Other approaches aim to directly embed the logical information of an OWL ontology in the latent space \cite{kulmanov2019embeddings,mondala2021emel++,peng2022description,xiong2022box}, mostly targeting the OWL 2 EL profile \cite{krotzsch2012OWL}, whose semantics are defined according to the DL \el\ \citep{baader2005pushing}. Prominent examples include 
ELEm \cite{kulmanov2019embeddings} and its extension \emelpp \cite{mondala2021emel++}, which model concepts as high-dimensional balls, but fail to faithfully capture concept conjunction, since the intersection of two balls is no longer a ball. This led to the development of the state-of-the-art methods BoxEL \cite{xiong2022box} and ELBE \cite{peng2022description}, which instead represent concepts as boxes (i.e., axis-aligned hyperrectangles).
However, all of these approaches still rely on TransE \cite{bordes2013translating} to model roles (i.e., binary relations) as simple translations, which is unable to capture one-to-many, many-to-one, or many-to-many relationships \cite{wang2014knowledge,lin2015learning}, and is limited in its ability to faithfully represent inclusion relationships between roles. 
Furthermore, these current works focus on the basic task of predicting subsumptions between named concepts, without considering complex concepts defined with logical operators or complex logical relationships in evaluation.

In this paper, we propose \boxsqel, a novel OWL ontology embedding method targeting the semantics of \el, which has been widely adopted in many real-life large-scale ontologies \citep{smith2007obo,hoehndorf2011common}.\footnote{The logical constructors provided by \el\ are very common. While some OWL ontologies use more complicated features not supported by \el, \boxsqel\ can still be used in that case to model the subset of axioms that fall into \el.} To address the aforementioned limitations of existing approaches, we instead draw inspiration from BoxE~\citep{abboud2020boxe} and represent both relations and concepts as boxes, while modeling interactions between concepts via a bumping mechanism.
We not only demonstrate how \boxsqel\ overcomes the shortcomings of previous methods, but also prove that it is \textit{sound}, i.e., faithfully captures the semantics of the underlying DL, which shows its theoretical correctness and supports interpretable inference for ontology completion. We evaluate our method in the two different inductive reasoning settings of concept subsumption prediction\,---\, involving both named and complex concepts\,---\, and role assertion (link) prediction, and on approximating deductive reasoning. Our results demonstrate that the theoretical advantages of our approach manifest themselves in practice and lead to state-of-the-art performance across a variety of datasets.

\section{Background and Related Work}
\subsection{Description Logic Ontologies}
A DL ontology $\ont$ describes some domain of interest in terms of individuals, concepts and roles, where individuals correspond to objects in the domain, concepts represent sets of objects, and roles are binary relations between objects. We limit our discussion to the DL \el~\citep{baader2005pushing}, which underpins the OWL 2 EL profile~\citep{krotzsch2012OWL}. It is widely adopted since it contains many useful and important knowledge representation features, while allowing for reasoning in polynomial time.
Given sets $\mathcal N_I$, $\mathcal N_C$ and $\mathcal N_R$ of, respectively, individual, concept, and role names, \el\ concepts are recursively defined as
\begin{equation*}
\top\;|\;\bot\;|\;A\;|\;C\sqcap D\;|\;\exists r.C\;|\;\{a\}\nonumber
\end{equation*}
where $\top$ is the top concept, $\bot$ is the bottom concept, $A \in \mathcal{N}_C$ is an atomic (or \textit{named}) concept, $r \in \mathcal{N}_R$ is an atomic role, $a \in \mathcal{N}_I$ is an individual, and $C$ and $D$ are themselves (possibly complex) \el\ concepts. We say a concept is \textit{complex} when it is constructed with a logical operator such as $\sqcap$ or $\exists$.
An \el{} ontology $\mathcal{O}$ consists of a TBox $\mathcal{T}$ and an ABox $\mathcal{A}$. The TBox consists of logical background knowledge in the form of concept subsumption axioms $C \sqsubseteq D$ and role inclusion axioms $r_1\circ \cdots \circ r_k \sqsubseteq r$, while the ABox contains concrete data in the form of concept and role assertion axioms $C(a)$ and $r(a, b)$. 
Note that a relational fact from a KG in the form of an RDF triple $(a, r, b)$ is equivalent to a role assertion axiom in the form of $r(a, b)$, and thus ontologies can be seen as extending KGs with more complex conceptual and logical information.

\begin{example}
\label{ex:family}
    The following ontology models a simple family domain:
    \begin{align*}
        \mathcal{T} = \{&\asf{Father}\sqsubseteq \asf{Parent}\sqcap \asf{Male},\;
        \asf{Mother}\sqsubseteq \asf{Parent}\sqcap \asf{Female},\\
        &\asf{Child}\sqsubseteq\exists \asf{hasParent}.\asf{Father},\;
        \asf{Child}\sqsubseteq\exists \asf{hasParent}.\asf{Mother},\\
        &\asf{hasParent}\sqsubseteq\asf{relatedTo}
        \}\\
        \mathcal{A} = \{&\asf{Father}(\asf{Alex}), \asf{Child}(\asf{Bob}), \asf{hasParent}(\asf{Bob}, \asf{Alex})\}
    \end{align*}
    The TBox specifies that a father is a male parent, a mother is a female parent, every child has a father and a mother, and having a parent implies being related to that parent; the ABox states that Alex is a father, Bob is a child, and Alex is a parent of Bob.
\end{example}

Similarly to first order logic, the semantics of \el\ are defined in terms of \textit{interpretations} that map individuals to elements, concept names to subsets, and role names to binary relations over some set called the \textit{interpretation domain}. An interpretation $\mathcal I$ that satisfies the semantics of every axiom in $\ont$ is called a \textit{model} of $\ont$, denoted as $\mathcal I \models \ont$. See Appendix A for a formal discussion of the semantics.

\paragraph{Remark.} \el\ also allows for so-called \textit{concrete domains} (a.k.a.\ datatypes and values), which we do not consider in this paper. Technically, we work on the $\mathcal{ELHO}(\circ)^\bot$ subset of \el.

\subsection{Subsumption Inference}
A central problem in DLs is to infer concept subsumptions from an ontology $\ont$. In general, these subsumptions can involve both named and complex concepts.
Classical reasoning algorithms leverage the logical information in $\ont$ to derive subsumptions that logically follow from the semantics; for example, we can infer from the ontology in \cref{ex:family} that $\asf{Child}\sqsubseteq\exists\asf{relatedTo}.\asf{Father}$. In contrast, \textit{inductive} reasoning (also called \textit{prediction}) aims to infer \textit{probable} subsumptions from $\ont$. Note that when we limit ourselves to predicting subsumptions of the form $\{a\}\sqsubseteq \exists r. \{b\}$, which are equivalent to role assertion axioms $r(a,b)$, this is identical to the problem of \textit{link prediction} in KGs (i.e., KG completion).

The majority of the existing prediction methods focus only on subsumptions between named concepts. Some approaches embed the formal semantics defined by the DL, while others focus on utilizing textual information such as concept labels and comments.
For the former, please see \cref{sec:rw_dle}. 
For the latter, OPA2Vec \cite{smaili2019opa2vec} and OWL2Vec* \cite{chen2021owl2vec} use a Word2Vec model trained with local graph structure augmented corpora to embed the text for predicting subsumptions, while the recent method BERTSubs \cite{chen2023contextual} fine-tunes a BERT model together with an attached classifier for predicting subsumptions involving both named concepts and complex concepts.  
Although these works consider a small part of the formal semantics such as the concept hierarchy as the context of a concept for augmenting prediction, they do not model the (complete) logical relationships of the ontology in the vector space. 
They are complementary to semantic ontology embedding methods including \boxsqel, but jointly embedding DL semantics and textual information is out of the scope of this paper.
 
\subsection{Knowledge Graph Embeddings}
KG embedding models such as TransE \cite{bordes2013translating}, DistMult \cite{yang2015embedding}, ComplEx \cite{trouillon2016Complex} and BoxE \cite{abboud2020boxe} aim to solve the problem of completing KGs composed of purely relational facts, and can be thought of as modeling only the role assertion part of the ABox of an OWL ontology \cite{wang2017knowledge}.
In particular, BoxE \cite{abboud2020boxe} also represents relations as boxes and adopts bump vectors, but it models relational facts alone, whereas we aim at much more complex DL ontologies with logical relationships involving concepts and roles.

Some KG embedding methods take background knowledge into account and are therefore related to ontology embedding techniques. However, these methods still focus only on modeling relational facts in a KG, using the background knowledge as constraints, and most of them only support logical rules concerning relations \cite{rocktaschel2015Injecting,wang2015Knowledge,guo2016Jointly,nayyeri2020Fantastic,nayyeri2021LogicENN} or schemas in simple languages like RDF Schema \cite{hao2019universal,xiang2021ontoea}. 
In contrast, ontology embedding methods including \boxsqel\ focus on OWL ontologies, which contain a large quantity of conceptual knowledge in the form of subsumptions and logical relationships. Furthermore, these methods can only be applied in the setting of link prediction, whereas ontology embeddings also allow predicting novel conceptual information or logical background knowledge itself.

\subsection{Semantic Ontology Embeddings}\label{sec:rw_dle}
Several ontology embedding methods for DL semantics have been proposed by learning geometric models.
ELEm \cite{kulmanov2019embeddings} is among the first to embed \el, and \emelpp \cite{mondala2021emel++} extends ELEm by considering role inclusion axioms. 
However, both methods represent concepts as high-dimensional balls, which have the disadvantage of not being closed under intersection.
Our concept representation based on boxes has previously been used in the two recent methods BoxEL \cite{xiong2022box} and ELBE \cite{peng2022description}. \citet{mondala2021emel++} is the only other technique we are aware of that also models role inclusion axioms; the other methods consider only a smaller subset of \el\@. All previous methods simply model roles (binary relations) by a single vector-based translation as in TransE~\citep{bordes2013translating}, which fails to faithfully capture one-to-many, many-to-one, and many-to-many relations \citep{wang2014knowledge,lin2015learning}. 
We use box-based modeling in combination with bump vectors to address the above problem, achieving better performance in ontology completion. 

Going beyond \el, \citet{ozcep2020Cone} introduce a cone-based model for the more expressive DL \alc; however, their contribution is mainly theoretical since they provide neither an implementation nor an evaluation. Embed2Reason~\citep{garg2019quantum} is another embedding approach for \alc\ ontologies based on quantum logic~\citep{birkhoff1936Logic}. In contrast to our work, its focus is on ABox instead of subsumption reasoning. 
Finally, it is worth mentioning that in comparison with all these DL embedding works, we conduct a more thorough evaluation by considering predicting subsumptions between not only named concepts, but also named concepts and complex concepts involving logical operators.

\begin{figure*}
    \centering
    \includegraphics[scale=0.5]{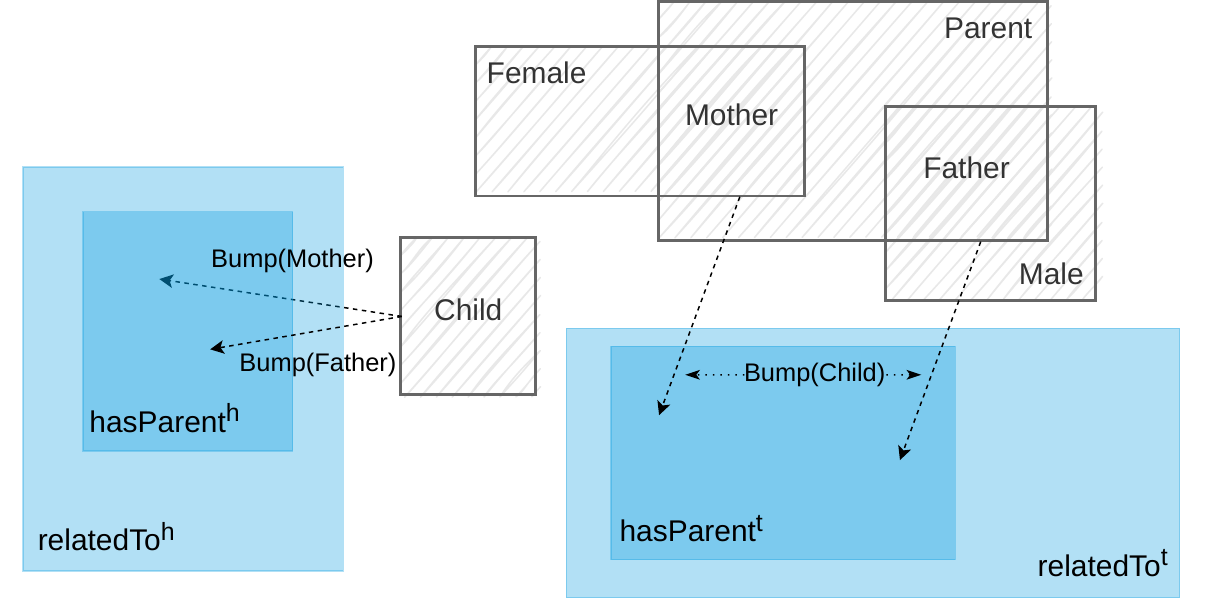}
    \caption{An illustration of \boxsqel\ embeddings. Striped boxes represent concept embeddings, whereas role embeddings are shaded blue and labelled as $\asf{r}^{\asf{h}}$ or \ $\asf{r}^{\asf{t}}$ for the head or tail box of $\asf r$, respectively. Bump vectors are drawn as arrows and labelled with the corresponding concept. The illustrated embeddings form a logical model of the TBox in \cref{ex:family}.}
    \label{fig:boxsqel}
    \Description{An illustration of BoxSquaredEL embeddings. The figure illustrates how the learned embeddings can capture the semantics of an ontology. For example, the embedding box for the concept Mother is exactly the intersection of the embeddings of Parent and Female.}
\end{figure*}

\section{Method}
In order to perform inductive reasoning over an \el\ ontology $\ont$ with signature $\Sigma = (\mathcal N_C, \mathcal N_R, \mathcal N_I)$, we follow the general framework of \citet{kulmanov2019embeddings} and learn embeddings that correspond to \textit{geometric models} of $\ont$; that is, (logical) models with interpretation domain $\Delta = \mathbb R^n$. We now specify how \boxsqel\ maps concepts, individuals, and roles to the embedding space $\mathbb R^n$ and describe the loss functions that encode the axioms of $\ont$.

\subsection{Geometric Representation}
\paragraph{Concepts and individuals.}
We follow recent work~\citep{peng2022description,xiong2022box} and represent concepts as $n$-dimensional \textit{boxes}, i.e., \textit{axis-aligned hyperrectangles}. This representation has several advantages over the alternative based on $n$-dimensional \textit{balls}, such as closure under concept intersection, and has been shown to work well in practice. Formally, we associate with every concept $C\in \mathcal N_C$ two vectors $\bm l_C\in\mathbb R^n$ and $\bm u_C\in\mathbb R^n$ such that $\bm l_C \leq \bm u_C$, where $\leq$~is applied element-wise. These vectors form the \textit{lower} and \textit{upper} corner of the box of $C$, i.e., $\boxx(C) = \{\,\bm x\in\mathbb R^n \mid \bm l_C \leq \bm x \leq \bm u_C \,\}$. The \textit{center} of $\boxx(C)$ is given by $(\bm l_C + \bm u_C) / 2$, and its \textit{offset} is $(\bm u_C - \bm l_C) / 2$.

We represent individuals $a\in \mathcal N_I$ as \textit{points} $\bm e_a \in \mathbb R^n$ in the embedding space. For notational convenience, nominals (e.g., $\{a\}$) are then formally mapped to boxes with volume 0, i.e., $\bm l_{\{a\}} = \bm u_{\{a\}} = \bm e_a$.

\paragraph{Roles.}
While most existing \el\ embedding models represent roles (binary relations) via simple translations in the form of TransE \citep{bordes2013translating}, we instead follow the idea of the BoxE KG embedding model~\citep{abboud2020boxe}.
That is, we associate every role $r\in\mathcal N_R$ with a \textit{head box} $\head(r)$ and a \textit{tail box} $\tail(r)$. Intuitively, every point in the head box is related via $r$ to every point in the tail box. This representation is made more expressive by introducing translational bump vectors $\bump(C) \in \mathbb R^n$ for every atomic concept $C$, which model interactions between related concepts by dynamically translating their embeddings. An axiom of the form $C\sqsubseteq\exists r.D$ is then considered to hold if
\begin{align}
\label{eq:roles}
    \begin{split}
        \boxx(C) \oplus \bump(D)&\subseteq\head(r)\quad\text{and}\\
        \boxx(D) \oplus \bump(C)&\subseteq\tail(r),
    \end{split}
\end{align}
where $\oplus$ denotes the operation of translating a box along a bump vector, i.e., $\boxx(C) \oplus \bump(D) = \{\, \bm x + \bump(D) \mid \bm x\in\boxx(C) \,\}$.

\begin{example}
     \cref{fig:boxsqel} illustrates 2-dimensional \boxsqel\ embeddings that form a logical model of the TBox in \cref{ex:family}, since, e.g., $\boxx(\asf{Father}) \subseteq \boxx(\asf{Parent}) \cap \boxx(\asf{Male})$ and \cref{eq:roles} holds for all relevant axioms. We will formalise the logical geometric model associated with a set of \boxsqel\ embeddings in \cref{theo:soundness}. Note also that the embeddings furthermore imply the subsumption $\asf{Child}\sqsubseteq\exists\asf{relatedTo}.\asf{Father}$ (again by \cref{eq:roles}), illustrating their utility for approximate reasoning.
\end{example}

\paragraph{Expressiveness.}
The previous example illustrates the expressive power of \boxsqel\@. Representing roles as head and tail boxes allows modeling one-to-many relationships such as \tsf{hasParent} faithfully, in contrast to previous approaches that employ a TransE-based role representation. Moreover, role inclusion axioms such as $\asf{hasParent}\sqsubseteq\asf{relatedTo}$ can be represented naturally via inclusion constraints on the relevant head and tail boxes. In contrast, previous methods either do not consider role inclusion axioms at all or only rudimentarily approximate them by forcing the embeddings of the involved roles to be similar~\citep{mondala2021emel++}.

\paragraph{Model complexity.} \boxsqel\ requires $2n|\mathcal N_C| + n|\mathcal N_I|$ parameters to store the lower and upper corners of the box embeddings of concepts and points associated with individuals. In order to represent the head and tail boxes for every relation and a bump vector per concept and individual, the model needs $4n|\mathcal N_R| + n(|\mathcal N_C| + |\mathcal N_I|)$ additional parameters. Therefore, the total space complexity of \boxsqel\ is in $O\Bigl(n\bigl(3|\mathcal N_C| + 2|\mathcal N_I| + 4|\mathcal N_R|\bigr)\Bigr)$.

\subsection{Training Procedure}\label{sec:training}
In order to learn embeddings for $\ont$, we first convert its ABox axioms into equivalent TBox axioms using the transformation rules
\begin{align*}
    C(a)\;&\rightsquigarrow\; \{a\}\sqsubseteq C\\
    r(a,b)\; &\rightsquigarrow\; \{a\}\sqsubseteq\exists r.\{b\}
\end{align*}
and then normalize the axioms using the standard procedure in \cite{baader2005pushing,baader2005Pushinga}. This results in a normalized ontology in which all axioms are in one of the normal forms described below. Crucially, this procedure is semantics-preserving: every model of the normalized ontology is also a model of the original ontology~\citep{baader2005pushing}. For more details on normalization, see Appendix B\@. We introduce separate loss functions for axioms in each normal form, which intuitively ensure that the learned embeddings adhere to the semantics of $\ont$. Finally, we minimize the sum of all loss terms $\loss(\mathcal O)$ via mini-batch gradient descent.

Our loss functions are based on the distance-based loss formulation in~\citep{kulmanov2019embeddings,peng2022description} 
and aim to minimize the element-wise distance between the embeddings of related concepts. Given two arbitrary boxes $A$ and $B$, this element-wise distance is computed as
\begin{equation*}
    \bm d(A, B) = \abs{\bm c(A) - \bm c(B)} - \bm o(A) - \bm o(B),
\end{equation*}
where $\bm c(\cdot)$ and $\bm o(\cdot)$ denote the center and offset of a box, respectively. Note that for nominals $\{a\}$ we have that $\bm c(\boxx(\{a\})) = \bm e_a$ and $\bm o(\boxx(\{a\})) = \bm 0$.

\paragraph{Generic inclusion loss.}
Before defining loss functions for the different \el\ normal forms, we first introduce a generic inclusion loss $\mathcal L_\subseteq(A,B)$. It encourages box $A$ to be contained in box $B$ and is defined as
\begin{equation*}
    \mathcal L_\subseteq(A, B) = 
    \begin{cases}
        \norm{\max\{\bm 0,\, \bm d(A, B) + 2\bm o(A) - \gamma\}} & \text{if } B \neq \emptyset\\
        \max\{0, \bm o(A)_1 + 1\} & \text{otherwise},
    \end{cases}
\end{equation*}
where $\gamma\in\mathbb R$ is a margin hyperparameter. If $\mathcal L_\subseteq(A, B) = 0$, either $A$ lies within $\gamma$-distance of $B$ in each dimension, or both $A$ and $B$ are empty.

We next introduce each normal form and the corresponding loss function. 
Note that all concepts in the normal forms below are atomic concepts or nominals (and not complex).

\paragraph{First normal form (NF1).}
For an NF1 axiom of the form $C\sqsubseteq D$, the learned embeddings need to satisfy $\boxx(C)\subseteq \boxx(D)$, corresponding to the semantics of concept inclusion. Therefore, we define the loss for NF1 as simply the inclusion loss:
\begin{equation*}
    \loss_{1}(C, D) = \loss_\subseteq(\boxx(C), \boxx(D)).
\end{equation*}

\paragraph{Second normal form (NF2).}
For an NF2 axiom of the form $C\sqcap D\sqsubseteq E$, we similarly require that the intersection of the boxes of $C$ and $D$ is within the box of $E$. The intersection of $\boxx(C)$ and $\boxx(D)$ is itself a box with lower corner $\max\{\bm l_C, \bm l_D\}$ and upper corner $\min\{\bm u_C, \bm u_D\}$, where max and min are applied element-wise. We thus have
\begin{equation*}
    \loss_{2}(C, D, E) = \loss_\subseteq\Bigl(\boxx(C)\cap \boxx(D), \boxx(E)\Bigr).
\end{equation*}
However, this formulation is problematic since it can be easily minimized to 0 by setting $\boxx(C)$ and $\boxx(D)$ to be disjoint. While disjoint embeddings for $C$ and $D$ would technically not violate the semantics, usually an axiom of the form $C\sqcap D\sqsubseteq\bot$ would have been used directly if it had been the intention that $C$ and $D$ should be disjoint. Therefore, we add the additional term
\begin{equation*}
    \norm{\max\left\{\bm 0,\,\max\{\bm l_{C}, \bm l_{D}\} - \min\{\bm u_{C},\bm u_{D}\}\right\}}
\end{equation*}
to the loss, which intuitively encourages $\boxx(C) \cap \boxx(D)$ to be non-empty by making all elements of its offset vector positive.

\paragraph{Third normal form (NF3).}
For an NF3 axiom of the form $C\sqsubseteq\exists r. D$, the embeddings should satisfy $\boxx(C) + \bump(D)\subseteq\head(r)$ and $\boxx(D) + \bump(C)\subseteq\tail(r)$. This is captured by the following loss function:
\begin{equation*}
    \begin{split}
        \loss_{3}(C, r, D) = \frac{1}{2}\Bigl(\loss_\subseteq(\boxx(C) + \bump(D), \head(r))\\ +\,\loss_\subseteq(\boxx(D) + \bump(C), \tail(r))\Bigr).
    \end{split}
\end{equation*}
If $\boxx(D)$ is empty, we furthermore add the term $\loss_\subseteq(\boxx(C),\emptyset)$ to also make $\boxx(C)$ empty.

\begin{figure}
    \centering
    \resizebox{\columnwidth}{!}{%
    \begin{minipage}{\columnwidth}
    \begin{align*}
    \asf{Father}&\sqsubseteq\asf{Male}\sqcap\asf{Parent}  &
    \asf{Mother}&\sqsubseteq\asf{Female}\sqcap\asf{Parent}\\
    \asf{Male}\sqcap\asf{Parent}&\sqsubseteq\asf{Father} &
    \asf{Female}\sqcap\asf{Parent}&\sqsubseteq\asf{Mother}\\
    \asf{Male}\sqcap\asf{Female}&\sqsubseteq\bot &
    \asf{Parent}\sqcap\asf{Child}&\sqsubseteq\bot\\
    \asf{Child}&\sqsubseteq\exists\asf{hasParent}.\asf{Mother} &
    \asf{Child}&\sqsubseteq\exists\asf{hasParent}.\asf{Father}\\
    \asf{Parent}&\sqsubseteq\exists\asf{hasChild}.\asf{Child}
    \end{align*}
    \end{minipage}
    }
    \caption{Proof of concept ontology.}
    \label{fig:family_ontology}
\end{figure}

\paragraph{Fourth normal form (NF4).}
For an NF4 axiom of the form $\exists r. C\sqsubseteq D$, we need to ensure that all points in the embedding space that are connected to $C$ via role $r$ are contained in $\boxx(D)$. It can be seen from our geometric representation that the set of these points is contained in the set $\head(r) - \bump(C)$. We therefore define the loss for the fourth normal form as
\begin{equation*}
    \loss_{4}(r, C, D) = \loss_{\subseteq}(\head(r) - \bump(C), \boxx(D)).
\end{equation*}

\paragraph{Fifth normal form (NF5).}
Axioms of the fifth normal form $C\sqcap D\sqsubseteq\bot$ state that concepts $C$ and $D$ are disjoint. Our corresponding loss function penalizes embeddings for which the element-wise distance is not greater than 0 (within a margin of $\gamma$) and is defined as
\begin{equation*}
    \loss_5(C, D) = \norm{\max\{\bm 0,\, -(\bm d(\boxx(C), \boxx(D)) + \gamma)\}}.
\end{equation*}

\paragraph{Role inclusion axioms.} After normalization, role inclusion axioms are either of the form $r\sqsubseteq s$ or $r_1\circ r_2\sqsubseteq s$. For the first case, we define the loss function
\begin{equation*}
    \loss_6(r, s) = \frac{1}{2}\Bigl(\loss_\subseteq(\head(r), \head(s))\\ +\,\loss_\subseteq(\tail(r), \tail(s))\Bigr),
\end{equation*}
which intuitively makes any embeddings related via role $r$ also related via $s$. Similarly, in the second case we have the loss
\begin{equation*}
    \loss_7(r_1, r_2, s) = \frac{1}{2}\Bigl(\loss_\subseteq(\head(r_1), \head(s))\\ +\,\loss_\subseteq(\tail(r_2), \tail(s))\Bigr).
\end{equation*}

\paragraph{Regularization.}
In order to prevent our expressive role representation from overfitting, we furthermore $L^2$-regularize the bump vectors by adding the regularization term
\begin{equation*}
    \lambda \sum_{C\in \mathcal N_C \cup \mathcal N_I} \norm{\bump(C)},
\end{equation*}
where $\lambda$ is a hyperparameter.

\paragraph{Negative sampling.}
In addition to the above loss functions, we also employ \textit{negative sampling} to further improve the quality of the learned embeddings. We follow previous work~\citep{kulmanov2019embeddings,peng2022description} and generate negative samples for axioms of the form $C\sqsubseteq\exists r. D$ by replacing either $C$ or $D$ with a randomly selected different concept, similar to negative sampling in KGs~\citep{bordes2013translating}. For every NF3 axiom we generate a new set of $\omega \geq 1$ negative samples $C\not\sqsubseteq \exists r.D$ in every epoch, for which we optimize the loss
\begin{align*}
    \loss_{\not\sqsubseteq}(C, r, D) ={}
    &\left(\delta - \mu(\boxx(C) + \bump(D),\, \head(r))\right)^2\\
    {}+{}&\left(\delta - \mu(\boxx(D) + \bump(C),\, \tail(r))\right)^2,
\end{align*}
where $\mu(A, B) = \norm{\max\{\bm 0,\, \bm d(A, B) + \gamma\}}$ is the minimal distance between any two points in $A$ and $B$ and $\delta > 0$ is a hyperparameter that controls how unlikely the negative samples are made by the model.

As in the case of KGs, the above procedure may occasionally generate false negative samples, i.e., negative axioms that actually do follow from $\ont$. However, on average it will predominantly produce true negatives, which we find to empirically improve the learned embeddings. 

\subsection{Soundness}
We now show that the embeddings learned by \boxsqel\ indeed correspond to a logical geometric model of the given ontology $\ont$.

\begin{theorem}[Soundness]
    \label{theo:soundness}
    Let $\mathcal O = (\mathcal T, \mathcal A)$ be an \el\ ontology. If $\gamma\leq 0$ and there exist \boxsqel\ embeddings in $\mathbb R^n$ such that $\loss(\mathcal O) = 0$, then these embeddings are a model of $\mathcal O$.
\end{theorem}

\begin{proofsketch}
The proof relies on the correctness of our loss functions, i.e., if $\loss_\subseteq(A, B) = 0$ then $A\subseteq B$. We first construct a geometric interpretation $\mathcal I$ of $\ont$ from the embeddings by interpreting every individual as its associated point, every concept as the set of points in its associated box, and every role as
the Cartesian product of its head box and its tail box. Since the loss for every axiom is 0, the embeddings satisfy the corresponding semantics. We thus have that $\mathcal I$ is a model of $\ont$.
\end{proofsketch}

A formal proof of \cref{theo:soundness} is given in Appendix C\@. While our optimization procedure might not achieve a loss of 0 in practice, the importance of \cref{theo:soundness} is that it demonstrates that the learned embeddings converge to a semantically meaningful representation of the ontology in which all of its axioms are satisfied, i.e., a model of $\ont$. The embeddings therefore indeed encode the axioms in $\ont$ and are thus useful to perform inductive or approximate reasoning.

\section{Evaluation}
We first validate our model and demonstrate its expressiveness on a proof of concept ontology.
We then evaluate \boxsqel\ on three tasks: general subsumption prediction, link (role assertion) prediction, and approximating deductive reasoning. Furthermore, we conduct a variety of ablation studies whose results are shown in Appendix~E\@. Our implementation is based on PyTorch~\citep{paszke2019pyTorch}, and we use the \textsf{jcel} reasoner~\citep{mendez2012jcel} to transform ontologies into normal form axioms.

\begin{figure}
    \centering
    \begin{subfigure}{.49\columnwidth}
        \centering
        \includegraphics[width=\textwidth]{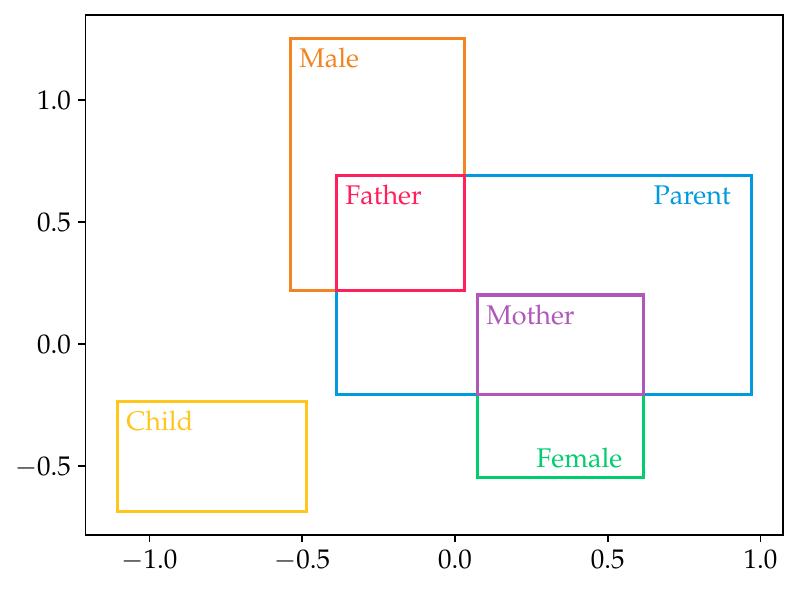}
        \caption{\boxsqel}
    \end{subfigure}
    \begin{subfigure}{.49\columnwidth}
        \centering
        \includegraphics[width=\textwidth]{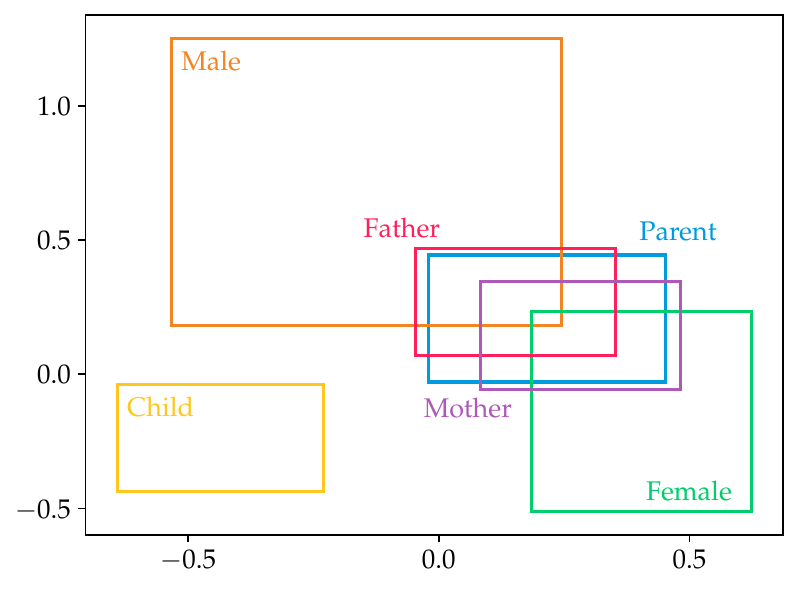}
        \caption{ELBE}
    \end{subfigure}
    \caption{Visualization of the embeddings learned by \boxsqel\ (left) and ELBE (right) for the proof-of-concept ontology. While \boxsqel\ can accurately represent the axioms in the ontology, the limitations of TransE as a model for roles prevent ELBE from learning correct embeddings.}
    \label{fig:family}
    \Description{The figure shows the embeddings learned by BoxSquaredEL and ELBE. The boxes representing the embeddings for BoxSquaredEL accurately match the semantics of the ontology, whereas the embeddings for ELBE are overlapping. For example, ELBE learns overlapping embeddings for Mother and Father.}
\end{figure}

\subsection{Proof of Concept: Family Ontology}
We visualize the embeddings learned by \boxsqel\ for a proof of concept family ontology given in \cref{fig:family_ontology}. To this end, we train \boxsqel\ with an embedding dimensionality of $n=2$, a margin of $\gamma=0$, regularization strength $\lambda=1$, and no negative sampling. We also train embeddings for ELBE~\citep{peng2022description}, a comparable \el\ ontology embedding model that similarly uses boxes to represent concepts, but interprets roles as translations. In order to ensure the volume of the learned embeddings is large enough for plotting, we add the following visualization loss term to the objective function of both models:
\begin{equation*}
    \loss_V = \frac{1}{n|\mathcal N_C|}\sum_{C\in\mathcal N_C}\sum_{i=1}^n \max\{0, 0.2 - \bm o(\boxx(C))_i\}.
\end{equation*}

The learned concept embeddings of both models are depicted in \cref{fig:family}. We see that \boxsqel\ is able to successfully learn embeddings that align with the axioms in the ontology. In particular, the embeddings fulfill all disjointness axioms and correctly represent the relationships between the concepts \tsf{Father}, \tsf{Male}, \tsf{Mother}, \tsf{Female}, and \tsf{Parent}.

\begin{table}[t]
      \centering
        \caption{Overall subsumption prediction results combined across normal forms.}
        \label{tab:combined}
        \resizebox{\columnwidth}{!}{%
        \begin{tabular}{@{}llrrrrrrr@{}}
        \toprule
                          & Model & H@1 & H@10       & H@100 & Med  & MRR  & MR   & AUC  \\ \midrule
        \multirow{5}{*}{\rotatebox[origin=c]{90}{GALEN}}   & ELEm  & 0.01   & 0.12          & 0.29     & 1662 & 0.05 & 5153 & 0.78 \\
        & \emelpp  & 0.01   & 0.11          & 0.24     & 2295 & 0.05 & 5486 & 0.76\\
        & BoxEL & 0.00 & 0.03	& 0.16	& 4785	& 0.01	& 7163	& 0.69\\
        & ELBE  & 0.02   & 0.14          & 0.27     & 1865 & 0.06 & 5303 & 0.77 \\
        & \boxsqel & \textbf{0.05} & \textbf{0.20} & \textbf{0.35} & \textbf{669} & \textbf{0.10} & \textbf{4375} & \textbf{0.81} \\ \midrule
        \multirow{5}{*}{\rotatebox[origin=c]{90}{GO}}      & ELEm  & 0.03   & \textbf{0.24} & 0.43     & 272  & 0.09 & 6204 & 0.86 \\
        & \emelpp  & 0.03   & 0.23          & 0.38     & 597 & 0.09 & 6710 & 0.85\\
        & BoxEL & 0.01	& 0.06	& 0.08	& 8443	& 0.03	& 14905	& 0.68\\
        & ELBE  & 0.01   & 0.10          & 0.22     & 1838 & 0.04 & 6986 & 0.85 \\
        & \boxsqel & \textbf{0.04} & 0.23          & \textbf{0.59} & \textbf{48}  & \textbf{0.10} & \textbf{3248} & \textbf{0.93} \\ \midrule
        \multirow{5}{*}{\rotatebox[origin=c]{90}{Anatomy}} & ELEm  & 0.10   & 0.40          & 0.64     & 22   & 0.19 & 6464 & 0.94 \\
        & \emelpp  & 0.11   & 0.36          & 0.57     & 36 & 0.19 & 8472 & 0.92\\
        & BoxEL & 0.03	& 0.12	& 0.28	& 1151	& 0.06	& 10916	& 0.90\\
        & ELBE  & 0.04   & 0.36          & 0.63     & 29   & 0.15 & 5400 & 0.95 \\
        & \boxsqel & \textbf{0.16} & \textbf{0.47} & \textbf{0.70} & \textbf{13}  & \textbf{0.26} & \textbf{2675} & \textbf{0.97} \\ \bottomrule
    \end{tabular}%
    }
\end{table}%

In contrast, we find that the embeddings learned by ELBE violate several of the axioms in the ontology. This is due to the inability of the underlying TransE model to correctly represent one-to-many relationships: because the ontology contains the axioms $\asf{Child}\sqsubseteq\exists\asf{hasParent}.\asf{Mother}$ as well as $\asf{Child}\sqsubseteq\exists\asf{hasParent}.\asf{Father}$, the model is forced to let the embeddings of \tsf{Mother} and \tsf{Father} overlap.

\subsection{General Subsumption Prediction}
We next evaluate \boxsqel\ on general subsumption prediction for inductive reasoning. In contrast to previous work~\citep{mondala2021emel++,xiong2022box}, we not only consider subsumptions between atomic (named) concepts, but also the more challenging task of predicting subsumptions between atomic concepts and complex concepts.

\paragraph{Benchmark.} We introduce a modification of the benchmark based on three biomedical ontologies GALEN~\citep{rector1996galen}, Gene Ontology~(GO)~\citep{ashburner2000gene} and Anatomy (a.k.a.\ Uberon)~\citep{mungall2012uberon}, that has been considered in previous DL embedding works \citep{mondala2021emel++,xiong2022box}. For each ontology, our benchmark consists of axioms split into training~(80\%), validation~(10\%), and testing~(10\%) sets for each normal form. This enables us to evaluate ontology embedding models on subsumption prediction between atomic concepts (NF1), atomic concepts and conjunctions (NF2), and atomic concepts and existential restrictions (NF3 and NF4). We report statistics on the sizes of these ontologies, including the numbers of axioms of different forms, in Table 7 in the appendix.

\paragraph{Baselines.} We compare \boxsqel\ with the state-of-the-art ontology embedding methods ELEm~\citep{kulmanov2019embeddings}, \emelpp~\citep{mondala2021emel++}, BoxEL~\citep{xiong2022box}, and ELBE~\citep{peng2022description}. We do not consider any traditional KG embedding methods in our experiments, since they have been shown to be considerably outperformed by ontology embedding methods~\citep{mondala2021emel++,xiong2022box} and are not applicable in the setting of complex concepts.

\paragraph{Evaluation protocol.} To evaluate the subsumption prediction performance, we follow the literature~\citep{mondala2021emel++,xiong2022box} and report a variety of ranking-based metrics on the testing set. Given a test axiom in some normal form, we generate a set of candidate predictions by replacing the atomic side of the subsumption with all the atomic concepts in $\mathcal N_C$. We then rank all candidate predictions by a score based on the distance between the embeddings of the concepts of the subsumption (for details see Appendix F) and record the rank of the true axiom. We report the standard metrics Hits@$k$ (H@$k$), where $k\in\{1, 10, 100\}$, the median rank (Med), the mean reciprocal rank (MRR), the mean rank (MR), and the area under the ROC curve (AUC). These metrics are computed for the axioms in each normal form individually, as well as combined across normal forms. See Appendix G for definitions of the metrics.

\begin{table}[t]
      \centering
        \caption{Detailed subsumption prediction results on the GALEN ontology.}
        \label{tab:galen}
        \resizebox{\columnwidth}{!}{%
        \begin{tabular}{@{}llrrrrrrr@{}}
        \toprule
                       & Model    & H@1        & H@10       & H@100      & Med           & MRR           & MR            & AUC           \\ \midrule
        \multirow{5}{*}{\rotatebox{90}{NF1}}      & ELEm     & 0.01          & 0.16          & 0.40          & 430           & 0.06          & 3568          & 0.85          \\
        & \emelpp     & 0.02 & 0.16          & 0.37          & 632           & 0.06 & 3765          & 0.84          \\
        & BoxEL & 0.00 & 0.00 & 0.05 & 3715 & 0.00 & 5727 & 0.75\\
        & ELBE     & \textbf{0.03} & 0.24          & 0.47          & 138           & 0.10 & \textbf{2444}          & \textbf{0.89}          \\
        & \boxsqel & \textbf{0.03}          & \textbf{0.30} & \textbf{0.51} & \textbf{91}   & \textbf{0.12}         & 2632 & \textbf{0.89} \\ \midrule
        \multirow{5}{*}{\rotatebox{90}{NF2}}      & ELEm     & 0.01          & 0.07          & 0.17          & 5106          & 0.03          & 7432          & 0.68          \\
        & \emelpp & 0.01 & 0.07 & 0.15 & 5750 & 0.03 & 7767 & 0.66\\
        & BoxEL & 0.00 & 0.00 & 0.00 & 11358 & 0.00 & 11605 & 0.50\\
        & ELBE     & 0.03          & 0.06          & 0.11          & 6476          & 0.04          & 8068          & 0.65          \\
        & \boxsqel & \textbf{0.06} & \textbf{0.15} & \textbf{0.28} & \textbf{2149} & \textbf{0.09} & \textbf{6265} & \textbf{0.73} \\ \midrule
        \multirow{5}{*}{\rotatebox{90}{NF3}}      & ELEm     & 0.02          & 0.14          & 0.28          & 1479          & 0.05          & 4831          & 0.79          \\
        & \emelpp & 0.02 & 0.11 & 0.22 & 2240 & 0.05 & 5348 & 0.77\\
        & BoxEL & 0.00 & 0.02 & 0.08 & 7239 & 0.01 & 8615 & 0.63\\
        & ELBE     & 0.03          & 0.14          & 0.25          & 2154          & 0.07          & 5072          & 0.78          \\
        & \boxsqel & \textbf{0.08} & \textbf{0.19} & \textbf{0.32} & \textbf{635} & \textbf{0.12} & \textbf{3798} & \textbf{0.84} \\ \midrule
        \multirow{5}{*}{\rotatebox{90}{NF4}}      & ELEm     & 0.00          & 0.05          & 0.18          & 3855          & 0.02 & 6793 & 0.71 \\
        & \emelpp & 0.00 & 0.04 & 0.12 & 4458 & 0.01 & 7020 & 0.70\\
        & BoxEL & 0.00 & \textbf{0.15} & \textbf{0.69} & \textbf{47} & \textbf{0.04} & \textbf{2667} & \textbf{0.89}\\
        & ELBE     & 0.00          & 0.03          & 0.07          & 7563          & 0.01          & 8884          & 0.62          \\
        & \boxsqel & 0.00          & 0.06 & 0.15 & 4364 & 0.02 & 7266          & 0.69 \\ \bottomrule
    \end{tabular}%
    }
\end{table} 
\begin{table}[t]
    \centering
    \caption{Detailed subsumption prediction results on the GO ontology.}
    \label{tab:go_detailed}
        \resizebox{\columnwidth}{!}{
        \begin{tabular}{@{}llrrrrrrr@{}}
        \toprule
                       & Model    & H@1 & H@10       & H@100      & Med  & MRR  & MR            & AUC           \\ \midrule
        \multirow{5}{*}{\rotatebox{90}{NF1}}      & ELEm     & 0.01   & 0.13          & 0.35          & 590  & 0.05 & 6433          & 0.86          \\
        & \emelpp & 0.01 & 0.12 & 0.30 & 1023 & 0.05 & 6709 & 0.85\\
        & BoxEL     & 0.00 & 0.01 & 0.05 & 5374 & 0.00 & 13413 & 0.71        \\
        & ELBE     & 0.01   & 0.10          & 0.24          & 1156 & 0.04 & 5657          & 0.88          \\
        & \boxsqel & \textbf{0.03} & \textbf{0.17} & \textbf{0.58} & \textbf{58} & \textbf{0.08} & \textbf{2686} & \textbf{0.94} \\ \midrule
        \multirow{5}{*}{\rotatebox{90}{NF2}}      & ELEm     & 0.12   & 0.49          & 0.63          & 11   & 0.24 & 4508          & 0.90          \\
        & \emelpp & 0.11 & 0.44 & 0.55 & 23 & 0.21 & 5169 & 0.89\\
        & BoxEL & 0.00 & 0.00 & 0.00 & 22882 & 0.00 & 23007 & 0.50\\
        & ELBE     & 0.01   & 0.05          & 0.09          & 6456 & 0.02 & 9421          & 0.80          \\
        & \boxsqel & \textbf{0.18} & \textbf{0.58} & \textbf{0.75} & \textbf{6}  & \textbf{0.31} & \textbf{2104} & \textbf{0.95} \\ \midrule
        \multirow{5}{*}{\rotatebox{90}{NF3}} & ELEm     & \textbf{0.06} & \textbf{0.40} & 0.52 & \textbf{54} & \textbf{0.15} & 6292          & 0.86          \\
        & \emelpp & 0.05 & 0.39 & 0.48 & 210 & \textbf{0.15} & 7788 & 0.83\\
        & BoxEL & 0.00 & 0.00 & 0.00 & 17027 & 0.00 & 18947 & 0.59\\
        & ELBE     & 0.02   & 0.15          & 0.30          & 959  & 0.07 & 7131          & 0.84          \\
        & \boxsqel & 0.00   & 0.18          & \textbf{0.53}          & 79   & 0.05 & \textbf{5042} & \textbf{0.89} \\ \midrule
        \multirow{5}{*}{\rotatebox{90}{NF4}} & ELEm     & 0.01 & 0.49 & 0.60          & \textbf{12} & 0.12 & 6272          & 0.86          \\
        & \emelpp & 0.01 & 0.49 & 0.58 & \textbf{12} & 0.13 & 6442 & 0.86\\
        & BoxEL & \textbf{0.09} & \textbf{0.54} & 0.54 & 2215 & \textbf{0.28} & 9673 & 0.79\\
        & ELBE     & 0.00   & 0.07          & 0.12          & 9049 & 0.02 & 12868         & 0.72          \\
        & \boxsqel & 0.00   & 0.37          & \textbf{0.64} & 20   & 0.08 & \textbf{4989} & \textbf{0.89} \\ \bottomrule
    \end{tabular}%
        }%
\end{table}

\paragraph{Experimental protocol.}
The embeddings are optimized with Adam \citep{kingma2015adam} for a maximum of 10,000 epochs. All hyperparameters are described in detail in Appendix H\@. We evaluate the models on a fraction of the validation set every $100$ epochs and choose the embeddings that achieve the best performance for final evaluation on the testing set. The results we report are averages across 5 different random seeds.

\paragraph{Results.}
The results on all the testing axioms (combined across all normal forms) are reported in \cref{tab:combined}. For detailed results on testing axioms of each normal form, see \cref{tab:galen,tab:go_detailed,tab:anatomy_detailed}.
We first observe that our model \boxsqel\ consistently outperforms all the baselines on all datasets, often with significant performance gains. For example, the median rank (MR) of \boxsqel\ is around 60\% lower than the second best-performing method on GALEN, more than 80\% lower on GO, and more than 40\% lower on Anatomy.
Among the baseline methods, results are similar; interestingly, ELEm generally performs best, in contrast to previous benchmarks.

From the detailed results, we observe that the novel role representation of \boxsqel\ not only generally improves prediction performance for NF3 axioms, which contain roles, but also for NF1 and NF2 axioms. This can be explained by the fact that the different normal forms are used to optimize the \textit{same} embeddings; i.e., if \boxsqel\ can better represent an axiom of the form $C\sqsubseteq\exists r. D$, it will learn better embeddings for $C$ and $D$, therefore also improving prediction quality for NF1 and NF2 axioms involving $C$ and/or $D$. There is no clear trend which axioms are the easiest to predict; on GALEN, the models generally perform better on NF1 axioms involving only atomic concepts, whereas on GO and Anatomy they perform similarly well on axioms involving complex concepts.

\begin{table}[t]
    \centering
    \caption{Detailed subsumption prediction results on the Anatomy ontology.}
    \label{tab:anatomy_detailed}
        \resizebox{\columnwidth}{!}{
            \begin{tabular}{@{}llrrrrrrr@{}}
        \toprule
        & Model    & H@1        & H@10       & H@100      & Med          & MRR           & MR            & AUC           \\ \midrule
        \multirow{5}{*}{\rotatebox{90}{NF1}}      & ELEm     & 0.07 & 0.30 & 0.57          & 43           & 0.14 & 9059          & 0.91          \\
        & \emelpp & \textbf{0.08} & 0.29 & 0.53 & 60 & 0.14 & 10414 & 0.90\\\
        & BoxEL & 0.01 & 0.05 & 0.16 & 1828 & 0.03 & 9597 & 0.91\\
        & ELBE     & 0.05          & 0.24          & 0.55          & 68           & 0.11          & 5177          & 0.95          \\
        & \boxsqel & 0.07          & \textbf{0.34}          & \textbf{0.65} & \textbf{27}  & \textbf{0.15}          & \textbf{2894} & \textbf{0.97} \\ \midrule
        \multirow{5}{*}{\rotatebox{90}{NF2}}      & ELEm     & 0.03          & 0.18          & 0.42          & 394          & 0.08          & 11592         & 0.89          \\
        & \emelpp & 0.03 & 0.18 & 0.35 & 1291 & 0.08 & 15759 & 0.85\\
        & BoxEL & 0.00 & 0.00 & 0.00 & 17607 & 0.00 & 26872 & 0.75\\
        & ELBE     & 0.02          & 0.11          & 0.26          & 1394         & 0.05          & 4885          & 0.96          \\
        & \boxsqel & \textbf{0.16} & \textbf{0.41} & \textbf{0.64} & \textbf{26}  & \textbf{0.24} & \textbf{1928} & \textbf{0.98} \\ \midrule
        \multirow{5}{*}{\rotatebox{90}{NF3}}      & ELEm     & 0.12          & 0.47          & 0.69          & 13           & 0.23          & 4686          & 0.96          \\
        & \emelpp & 0.13 & 0.42 & 0.60 & 23 & 0.23 & 7097 & 0.93\\
        & BoxEL & 0.04 & 0.17 & 0.36 & 567 & 0.08 & 11095 & 0.90\\
        & ELBE     & 0.04          & 0.44          & 0.70          & 16           & 0.18          & 5408          & 0.95          \\
        & \boxsqel & \textbf{0.21} & \textbf{0.56} & \textbf{0.75} & \textbf{7}   & \textbf{0.33} & \textbf{2466} & \textbf{0.98} \\ \midrule
        \multirow{5}{*}{\rotatebox{90}{NF4}}      & ELEm     & 0.00          & 0.03          & \textbf{0.23}          & \textbf{813}          & \textbf{0.01}          & 10230         & 0.91          \\
        & \emelpp & 0.00 & 0.02 & 0.17 & 1470 & \textbf{0.01} & 10951 & 0.90\\
        & BoxEL & 0.00 & 0.00 & 0.00 & 38942 & 0.00 & 41283 & 0.61\\
        & ELBE     & 0.00          & 0.02          & 0.06          & 6261         & \textbf{0.01}          & 15187         & 0.86          \\
        & \boxsqel & 0.00          & \textbf{0.05} & 0.14 & 3065 & \textbf{0.01} & \textbf{8366} & \textbf{0.92} \\ \bottomrule
    \end{tabular}%
        }%
\end{table}

\subsection{Link Prediction}
We next evaluate our model on the task of link prediction, i.e., predicting role assertions of the form $r(a,b)$, which is implemented by predicting subsumptions of the form $\{a\}\sqsubseteq\exists r.\{b\}$.

\paragraph{Datasets.}
We consider a real-world protein-protein interaction (PPI) prediction task introduced in~\citep{kulmanov2019embeddings}. They provide two ontologies for human and yeast organisms, constructed by combining the STRING database of PPIs~\citep{szklarczyk2021string} with the Gene Ontology~\citep{ashburner2000gene}. The proteins and their interactions recorded in STRING constitute the ABox, while GO acts as the TBox, and additional axioms modeling the association of proteins with their biological functions are added to the ontology. The task is to predict missing subsumptions of the form $\{P_1\}\sqsubseteq\exists\asf{interacts}.\{P_2\}$, where $P_1$ and $P_2$ represent two proteins.

\begin{table*}[t]
    \centering
    \caption{PPI prediction results on the yeast and human datasets. Columns annotated with~(F) contain filtered metrics, other columns contain raw metrics. The results for BoxEL are from~\protect\citep{xiong2022box}; all other baseline results are from~\protect\citep{peng2022description}.}
    \label{tab:ppi}
    \begin{tabular}{@{}llrrrrrrrr@{}}
        \toprule
         &
        Model &
        H@10 &
        H@10 (F) &
        H@100 &
        H@100 (F) &
        MR &
        MR (F) &
        AUC &
        AUC (F) \\ \midrule
        \multirow{5}{*}{\rotatebox{90}{Yeast}} &
        ELEm &
        0.10 &
        0.23 &
        0.50 &
        0.75 &
        247 &
        187 &
        0.96 &
        0.97 \\
        &
        \emelpp &
        0.08 &
        0.17 &
        0.48 &
        0.65 &
        336 &
        291 &
        0.94 &
        0.95 \\
        &
        BoxEL &
        0.09 &
        0.20 &
        0.52 &
        0.73 &
        423 &
        379 &
        0.93 &
        0.94 \\
        &
        ELBE &
        \textbf{0.11} &
        0.26 &
        0.57 &
        0.77 &
        201 &
        154 &
        0.96 &
        0.97 \\
        &
        \boxsqel &
        \textbf{0.11} &
        \textbf{0.33} &
        \textbf{0.64} &
        \textbf{0.87} &
        \textbf{168} &
        \textbf{118} &
        \textbf{0.97} &
        \textbf{0.98} \\ \midrule
        \multirow{5}{*}{\rotatebox{90}{Human}} &
        ELEm &
        \textbf{0.09} &
        0.22 &
        0.43 &
        0.70 &
        658 &
        572 &
        0.96 &
        0.96 \\
        &
        \emelpp &
        0.04 &
        0.13 &
        0.38 &
        0.56 &
        772 &
        700 &
        0.95 &
        0.95 \\
        &
        BoxEL &
        0.07 &
        0.10 &
        0.42 &
        0.63 &
        1574 &
        1530 &
        0.93 &
        0.93 \\
        &
        ELBE &
        \textbf{0.09} &
        0.22 &
        0.49 &
        0.72 &
        434 &
        362 &
        0.97 &
        \textbf{0.98} \\
        &
        \boxsqel &
        \textbf{0.09} &
        \textbf{0.28} &
        \textbf{0.55} &
        \textbf{0.83} &
        \textbf{343} &
        \textbf{269} &
        \textbf{0.98} &
        \textbf{0.98} \\ \bottomrule
    \end{tabular}%
\end{table*}

\paragraph{Baselines.}
We also consider ELEm~\citep{kulmanov2019embeddings}, \emelpp~\citep{mondala2021emel++}, BoxEL~\citep{xiong2022box}, and ELBE~\citep{peng2022description} for the baselines as in general subsumption prediction, and report the relevant best results from their original papers.

\paragraph{Evaluation and experimental protocol.}
In order to evaluate our method, we use the 80\%/10\%/10\% training, validation, and testing split of the PPI data provided by~\citep{kulmanov2019embeddings}. We compute the same ranking-based metrics as in subsumption prediction, both in the standard and filtered fashion, in which any true candidate predictions except for the target axiom to predict are first removed from the set of all candidate predictions before computing the ranks. The experimental protocol is the same as in subsumption prediction.

\begin{table}[t]
    \centering
    \caption{Deductive reasoning results on GALEN, GO, and Anatomy.}
    \label{tab:reasoning}
    \resizebox{\columnwidth}{!}{
    \begin{tabular}{@{}llrrrrrrr@{}}
        \toprule
                          & Model    & H@1 & H@10       & H@100      & Med          & MRR           & MR            & AUC           \\ \midrule
        \multirow{5}{*}{\rotatebox[origin=c]{90}{GALEN}}   & ELEm     & 0.00   & 0.04          & 0.20          & 1807         & 0.01          & 4405          & 0.81          \\
        & \emelpp & 0.00 & 0.04 & 0.18 & 2049 & 0.01 & 4634 & 0.81\\
        & BoxEL & 0.00 & 0.00 & 0.01 & 6906 & 0.00 & 7925 & 0.67\\
        & ELBE     & 0.00   & 0.06          & 0.16          & 1785         & 0.02          & 3974          & 0.84          \\
        & \boxsqel & \textbf{0.01} & \textbf{0.09} & \textbf{0.24} & \textbf{1003} & \textbf{0.03} & \textbf{2833} & \textbf{0.88} \\ \midrule
        \multirow{5}{*}{\rotatebox[origin=c]{90}{GO}}      & ELEm     & 0.00   & 0.04          & 0.22          & 1629         & 0.02          & 7377          & 0.84          \\
        & \emelpp & 0.00 & 0.04 & 0.19 & 1346 & 0.01 & 6557 & 0.86\\
        & BoxEL & 0.00 & 0.00 & 0.13 & 1085 & 0.00 & 5359 & 0.88\\
        & ELBE     & 0.00   & 0.06          & 0.21          & 935          & 0.02          & 3846          & 0.92          \\
        & \boxsqel & 0.00   & \textbf{0.08} & \textbf{0.49} & \textbf{107} & \textbf{0.04} & \textbf{1689} & \textbf{0.96} \\ \midrule
        \multirow{5}{*}{\rotatebox[origin=c]{90}{Anatomy}} & ELEm     & 0.00   & 0.07          & 0.28          & 901          & 0.02          & 7958          & 0.93          \\
        & \emelpp & 0.00 & 0.07 & 0.26 & 1576 & 0.02 & 10976 & 0.90\\
        & BoxEL & \textbf{0.01} & \textbf{0.10} & 0.24 & 838 & \textbf{0.04} & 9156 & 0.92\\
        & ELBE     & 0.00   & 0.08          & 0.32          & 336          & 0.03          & 2312          & 0.98          \\
        & \boxsqel & \textbf{0.01}   & 0.09 & \textbf{0.44} & \textbf{152} & \textbf{0.04} & \textbf{1599} & \textbf{0.99} \\ \bottomrule
    \end{tabular}
    }
\end{table}

\paragraph{Results.}
\cref{tab:ppi} lists the results of \boxsqel\ and the baseline methods on the yeast and human PPI prediction datasets. \boxsqel\ outperforms all the baselines, with significant performance gains for both datasets on most of the metrics including filtered hits and mean rank. The AUC and AUC (F) values are all very close to the maximum value $1.0$, but \boxsqel\ still improves the state-of-the-art on Yeast, and ties ELBE on Human.
All these results demonstrate the effectiveness of \boxsqel\ in role assertion prediction for ontologies with an ABox. 
The comparatively stronger results than ELBE and BoxEL, which share the same concept representation as our model, once again highlight the positive impact of our novel approach of representing the semantics of roles with boxes and bump vectors.

\subsection{Approximating Deductive Reasoning}
We finally investigate how well our model can approximate \textit{deductive reasoning}, i.e., infer subsumptions that are logical consequences of the axioms in the ontology.

\paragraph{Experimental setup.}
We again consider the GALEN, GO, and Anatomy ontologies. Instead of splitting the axioms into separate training, validation, and testing sets, we now train the models on the entire ontology using all of its asserted axioms. For evaluation, we use an \el\ reasoner to compute the complete set of NF1 axioms (i.e., atomic subsumptions) that are logically implied by, but not explicitly asserted in the given ontology. 
We split off 10\% of these implied NF1 axioms for the validation set and keep the remainder as the testing set.
We report the results of \boxsqel\ and the same baseline methods considered in subsumption prediction. The evaluation and experimental protocol is also the same as in subsumption prediction.\footnote{Deductive reasoning with \el\ embeddings has been previously considered in~\citep{mondala2021emel++}. However, we find that there is significant leakage (overlap) between their testing and training sets. We therefore do not adopt their reported results, but instead reproduce the results of the baseline models.}

\paragraph{Results.}
\cref{tab:reasoning} lists the results of approximating deductive reasoning. The baselines perform similarly, with ELBE achieving slightly stronger results than the others on GO and Anatomy. \boxsqel\ outperforms the baselines on almost all metrics across the three ontologies, with significant performance gains especially for Hits@$100$, median rank, and mean rank.
This indicates that \boxsqel\ is able to preserve more of the logical structure than the other embedding methods.

\paragraph{Deductive vs inductive reasoning.}
Comparing the results in \cref{tab:reasoning,tab:combined}, we observe that the embedding models generally perform better in subsumption prediction than in approximating deductive reasoning. To see why, note that the learned embeddings are used to make purely statistical predictions about missing axioms. The soundness of our method guarantees that these predictions align with the semantics of the ontology. However, we do not explicitly perform logical inference steps in the embedding space, as would be required to derive logical inferences similar to a deductive reasoning algorithm. We illustrate this difference with a concrete example in Appendix I\@. While embedding methods can thus be useful to approximate deductive reasoning, the two approaches are best used in conjunction in order to combine formal derivations with inductive and probable knowledge.

\section{Conclusion}
We developed \boxsqel, a novel OWL ontology embedding method that adopts box-based representations for both concepts and roles. This representation is able to model complex logical constructs from \el\ and overcomes the limitations of previous approaches in representing roles and role inclusion axioms.
We formally proved that our method is sound, i.e., correctly represents the semantics of \el, and performed an extensive empirical evaluation, achieving state-of-the-art results in concept subsumption prediction, role assertion prediction, and approximating deductive reasoning.


\begin{acks}
    This work was supported by Samsung Research UK (SRUK) and the EPSRC projects ConCur (EP/V050869/1), OASIS (EP/S032347/1), and UK FIRES (EP/S019111/1).
\end{acks}

\bibliographystyle{ACM-Reference-Format}
\balance
\bibliography{references}


\begin{thebibliography}{46}


\ifx \showCODEN    \undefined \def \showCODEN     #1{\unskip}     \fi
\ifx \showDOI      \undefined \def \showDOI       #1{#1}\fi
\ifx \showISBNx    \undefined \def \showISBNx     #1{\unskip}     \fi
\ifx \showISBNxiii \undefined \def \showISBNxiii  #1{\unskip}     \fi
\ifx \showISSN     \undefined \def \showISSN      #1{\unskip}     \fi
\ifx \showLCCN     \undefined \def \showLCCN      #1{\unskip}     \fi
\ifx \shownote     \undefined \def \shownote      #1{#1}          \fi
\ifx \showarticletitle \undefined \def \showarticletitle #1{#1}   \fi
\ifx \showURL      \undefined \def \showURL       {\relax}        \fi
\providecommand\bibfield[2]{#2}
\providecommand\bibinfo[2]{#2}
\providecommand\natexlab[1]{#1}
\providecommand\showeprint[2][]{arXiv:#2}

\bibitem[Abboud et~al\mbox{.}(2020)]%
        {abboud2020boxe}
\bibfield{author}{\bibinfo{person}{Ralph Abboud}, \bibinfo{person}{{\.I}smail~{\.I}lkan Ceylan}, \bibinfo{person}{Thomas Lukasiewicz}, {and} \bibinfo{person}{Tommaso Salvatori}.} \bibinfo{year}{2020}\natexlab{}.
\newblock \showarticletitle{{{BoxE}}: A Box Embedding Model for Knowledge Base Completion}. In \bibinfo{booktitle}{\emph{Advances in {{Neural Information Processing Systems}}}}, Vol.~\bibinfo{volume}{33}. \bibinfo{pages}{9649--9661}.
\newblock


\bibitem[Ashburner et~al\mbox{.}(2000)]%
        {ashburner2000gene}
\bibfield{author}{\bibinfo{person}{Michael Ashburner}, \bibinfo{person}{Catherine~A Ball}, \bibinfo{person}{Judith~A Blake}, \bibinfo{person}{David Botstein}, \bibinfo{person}{Heather Butler}, \bibinfo{person}{J~Michael Cherry}, \bibinfo{person}{Allan~P Davis}, \bibinfo{person}{Kara Dolinski}, \bibinfo{person}{Selina~S Dwight}, \bibinfo{person}{Janan~T Eppig}, {et~al\mbox{.}}} \bibinfo{year}{2000}\natexlab{}.
\newblock \showarticletitle{Gene ontology: tool for the unification of biology}.
\newblock \bibinfo{journal}{\emph{Nature genetics}} \bibinfo{volume}{25}, \bibinfo{number}{1} (\bibinfo{year}{2000}), \bibinfo{pages}{25--29}.
\newblock


\bibitem[Baader et~al\mbox{.}(2005a)]%
        {baader2005pushing}
\bibfield{author}{\bibinfo{person}{Franz Baader}, \bibinfo{person}{Sebastian Brandt}, {and} \bibinfo{person}{Carsten Lutz}.} \bibinfo{year}{2005}\natexlab{a}.
\newblock \showarticletitle{Pushing the {EL} envelope}. In \bibinfo{booktitle}{\emph{Proceedings of the Nineteenth International Joint Conference on Artificial Intelligence}}, Vol.~\bibinfo{volume}{5}. \bibinfo{pages}{364--369}.
\newblock


\bibitem[Baader et~al\mbox{.}(2005b)]%
        {baader2005Pushinga}
\bibfield{author}{\bibinfo{person}{Franz Baader}, \bibinfo{person}{Sebastian Brandt}, {and} \bibinfo{person}{Carsten Lutz}.} \bibinfo{year}{2005}\natexlab{b}.
\newblock \bibinfo{booktitle}{\emph{Pushing the {{EL}} Envelope}}.
\newblock \bibinfo{type}{{{LTCS-Report}}} LTCS-05-01. \bibinfo{institution}{{Institute for Theoretical Computer Science, TU Dresden}}.
\newblock
\urldef\tempurl%
\url{http://lat.inf.tu-dresden.de/research/reports.html}
\showURL{%
\tempurl}


\bibitem[Baader et~al\mbox{.}(2005c)]%
        {baader2005description}
\bibfield{author}{\bibinfo{person}{Franz Baader}, \bibinfo{person}{Ian Horrocks}, {and} \bibinfo{person}{Ulrike Sattler}.} \bibinfo{year}{2005}\natexlab{c}.
\newblock \showarticletitle{Description logics as ontology languages for the semantic web}.
\newblock In \bibinfo{booktitle}{\emph{Mechanizing mathematical reasoning}}. \bibinfo{publisher}{Springer}, \bibinfo{pages}{228--248}.
\newblock


\bibitem[Birkhoff and Von~Neumann(1936)]%
        {birkhoff1936Logic}
\bibfield{author}{\bibinfo{person}{Garrett Birkhoff} {and} \bibinfo{person}{John Von~Neumann}.} \bibinfo{year}{1936}\natexlab{}.
\newblock \showarticletitle{The {{Logic}} of {{Quantum Mechanics}}}.
\newblock \bibinfo{journal}{\emph{Annals of Mathematics}} \bibinfo{volume}{37}, \bibinfo{number}{4} (\bibinfo{year}{1936}), \bibinfo{pages}{823--843}.
\newblock
\showISSN{0003-486X}
\showeprint[jstor]{1968621}


\bibitem[Bordes et~al\mbox{.}(2013)]%
        {bordes2013translating}
\bibfield{author}{\bibinfo{person}{Antoine Bordes}, \bibinfo{person}{Nicolas Usunier}, \bibinfo{person}{Alberto {Garcia-Duran}}, \bibinfo{person}{Jason Weston}, {and} \bibinfo{person}{Oksana Yakhnenko}.} \bibinfo{year}{2013}\natexlab{}.
\newblock \showarticletitle{Translating Embeddings for Modeling Multi-relational Data}. In \bibinfo{booktitle}{\emph{Advances in {{Neural Information Processing Systems}}}}, Vol.~\bibinfo{volume}{26}.
\newblock


\bibitem[Chen et~al\mbox{.}(2023)]%
        {chen2023contextual}
\bibfield{author}{\bibinfo{person}{Jiaoyan Chen}, \bibinfo{person}{Yuan He}, \bibinfo{person}{Yuxia Geng}, \bibinfo{person}{Ernesto Jim{\'e}nez-Ruiz}, \bibinfo{person}{Hang Dong}, {and} \bibinfo{person}{Ian Horrocks}.} \bibinfo{year}{2023}\natexlab{}.
\newblock \showarticletitle{Contextual semantic embeddings for ontology subsumption prediction}.
\newblock \bibinfo{journal}{\emph{World Wide Web}} (\bibinfo{year}{2023}), \bibinfo{pages}{1--23}.
\newblock


\bibitem[Chen et~al\mbox{.}(2021)]%
        {chen2021owl2vec}
\bibfield{author}{\bibinfo{person}{Jiaoyan Chen}, \bibinfo{person}{Pan Hu}, \bibinfo{person}{Ernesto Jimenez-Ruiz}, \bibinfo{person}{Ole~Magnus Holter}, \bibinfo{person}{Denvar Antonyrajah}, {and} \bibinfo{person}{Ian Horrocks}.} \bibinfo{year}{2021}\natexlab{}.
\newblock \showarticletitle{{OWL2Vec*}: Embedding of {OWL} ontologies}.
\newblock \bibinfo{journal}{\emph{Machine Learning}} \bibinfo{volume}{110}, \bibinfo{number}{7} (\bibinfo{year}{2021}), \bibinfo{pages}{1813--1845}.
\newblock


\bibitem[Dooley et~al\mbox{.}(2018)]%
        {dooley2018foodon}
\bibfield{author}{\bibinfo{person}{Damion~M. Dooley}, \bibinfo{person}{Emma~J. Griffiths}, \bibinfo{person}{Gurinder~S. Gosal}, \bibinfo{person}{Pier~L. Buttigieg}, \bibinfo{person}{Robert Hoehndorf}, \bibinfo{person}{Matthew~C. Lange}, \bibinfo{person}{Lynn~M. Schriml}, \bibinfo{person}{Fiona S.~L. Brinkman}, {and} \bibinfo{person}{William W.~L. Hsiao}.} \bibinfo{year}{2018}\natexlab{}.
\newblock \showarticletitle{{{FoodOn}}: A Harmonized Food Ontology to Increase Global Food Traceability, Quality Control and Data Integration}.
\newblock \bibinfo{journal}{\emph{npj Science of Food}} \bibinfo{volume}{2}, \bibinfo{number}{1} (\bibinfo{date}{Dec.} \bibinfo{year}{2018}), \bibinfo{pages}{23}.
\newblock
\showISSN{2396-8370}


\bibitem[Garg et~al\mbox{.}(2019)]%
        {garg2019quantum}
\bibfield{author}{\bibinfo{person}{Dinesh Garg}, \bibinfo{person}{Shajith Ikbal}, \bibinfo{person}{Santosh~K Srivastava}, \bibinfo{person}{Harit Vishwakarma}, \bibinfo{person}{Hima Karanam}, {and} \bibinfo{person}{L~Venkata Subramaniam}.} \bibinfo{year}{2019}\natexlab{}.
\newblock \showarticletitle{Quantum embedding of knowledge for reasoning}.
\newblock \bibinfo{journal}{\emph{Advances in Neural Information Processing Systems}}  \bibinfo{volume}{32} (\bibinfo{year}{2019}).
\newblock


\bibitem[Glimm et~al\mbox{.}(2014)]%
        {glimm2014HermiT}
\bibfield{author}{\bibinfo{person}{Birte Glimm}, \bibinfo{person}{Ian Horrocks}, \bibinfo{person}{Boris Motik}, \bibinfo{person}{Giorgos Stoilos}, {and} \bibinfo{person}{Zhe Wang}.} \bibinfo{year}{2014}\natexlab{}.
\newblock \showarticletitle{{{HermiT}}: {{An OWL}} 2 {{Reasoner}}}.
\newblock \bibinfo{journal}{\emph{Journal of Automated Reasoning}} \bibinfo{volume}{53}, \bibinfo{number}{3} (\bibinfo{year}{2014}), \bibinfo{pages}{245--269}.
\newblock
\showISSN{1573-0670}


\bibitem[Grau et~al\mbox{.}(2008)]%
        {grau2008OWL}
\bibfield{author}{\bibinfo{person}{Bernardo~Cuenca Grau}, \bibinfo{person}{Ian Horrocks}, \bibinfo{person}{Boris Motik}, \bibinfo{person}{Bijan Parsia}, \bibinfo{person}{Peter {Patel-Schneider}}, {and} \bibinfo{person}{Ulrike Sattler}.} \bibinfo{year}{2008}\natexlab{}.
\newblock \showarticletitle{{{OWL}} 2: {{The}} next Step for {{OWL}}}.
\newblock \bibinfo{journal}{\emph{Journal of Web Semantics}} \bibinfo{volume}{6}, \bibinfo{number}{4} (\bibinfo{year}{2008}), \bibinfo{pages}{309--322}.
\newblock
\showISSN{1570-8268}


\bibitem[Guo et~al\mbox{.}(2016)]%
        {guo2016Jointly}
\bibfield{author}{\bibinfo{person}{Shu Guo}, \bibinfo{person}{Quan Wang}, \bibinfo{person}{Lihong Wang}, \bibinfo{person}{Bin Wang}, {and} \bibinfo{person}{Li Guo}.} \bibinfo{year}{2016}\natexlab{}.
\newblock \showarticletitle{Jointly {{Embedding Knowledge Graphs}} and {{Logical Rules}}}. In \bibinfo{booktitle}{\emph{Proceedings of the 2016 {{Conference}} on {{Empirical Methods}} in {{Natural}} {{Language Processing}}}}. \bibinfo{pages}{192--202}.
\newblock


\bibitem[Hao et~al\mbox{.}(2019)]%
        {hao2019universal}
\bibfield{author}{\bibinfo{person}{Junheng Hao}, \bibinfo{person}{Muhao Chen}, \bibinfo{person}{Wenchao Yu}, \bibinfo{person}{Yizhou Sun}, {and} \bibinfo{person}{Wei Wang}.} \bibinfo{year}{2019}\natexlab{}.
\newblock \showarticletitle{Universal representation learning of knowledge bases by jointly embedding instances and ontological concepts}. In \bibinfo{booktitle}{\emph{Proceedings of the 25th ACM SIGKDD International Conference on Knowledge Discovery \& Data Mining}}. \bibinfo{pages}{1709--1719}.
\newblock


\bibitem[H{\'{e}}der(2014)]%
        {heder2014semantic}
\bibfield{author}{\bibinfo{person}{Mih{\'{a}}ly H{\'{e}}der}.} \bibinfo{year}{2014}\natexlab{}.
\newblock \showarticletitle{\emph{Semantic Web for the Working Ontologist, Second edition: Effective modeling in {RDFS} and OWL}}.
\newblock \bibinfo{journal}{\emph{Knowl. Eng. Rev.}} \bibinfo{volume}{29}, \bibinfo{number}{3} (\bibinfo{year}{2014}), \bibinfo{pages}{404--405}.
\newblock


\bibitem[Hoehndorf et~al\mbox{.}(2011)]%
        {hoehndorf2011common}
\bibfield{author}{\bibinfo{person}{Robert Hoehndorf}, \bibinfo{person}{Michel Dumontier}, \bibinfo{person}{Anika Oellrich}, \bibinfo{person}{Sarala Wimalaratne}, \bibinfo{person}{Dietrich {Rebholz-Schuhmann}}, \bibinfo{person}{Paul Schofield}, {and} \bibinfo{person}{Georgios~V. Gkoutos}.} \bibinfo{year}{2011}\natexlab{}.
\newblock \showarticletitle{A Common Layer of Interoperability for Biomedical Ontologies Based on {{OWL EL}}}.
\newblock \bibinfo{journal}{\emph{Bioinformatics}} \bibinfo{volume}{27}, \bibinfo{number}{7} (\bibinfo{date}{April} \bibinfo{year}{2011}), \bibinfo{pages}{1001--1008}.
\newblock
\showISSN{1367-4803}


\bibitem[Hogan et~al\mbox{.}(2021)]%
        {hogan2021knowledge}
\bibfield{author}{\bibinfo{person}{Aidan Hogan}, \bibinfo{person}{Eva Blomqvist}, \bibinfo{person}{Michael Cochez}, \bibinfo{person}{Claudia d’Amato}, \bibinfo{person}{Gerard~de Melo}, \bibinfo{person}{Claudio Gutierrez}, \bibinfo{person}{Sabrina Kirrane}, \bibinfo{person}{Jos{\'e} Emilio~Labra Gayo}, \bibinfo{person}{Roberto Navigli}, \bibinfo{person}{Sebastian Neumaier}, {et~al\mbox{.}}} \bibinfo{year}{2021}\natexlab{}.
\newblock \showarticletitle{Knowledge graphs}.
\newblock \bibinfo{journal}{\emph{ACM Computing Surveys (CSUR)}} \bibinfo{volume}{54}, \bibinfo{number}{4} (\bibinfo{year}{2021}), \bibinfo{pages}{1--37}.
\newblock


\bibitem[Kazakov et~al\mbox{.}(2014)]%
        {kazakov2014Incredible}
\bibfield{author}{\bibinfo{person}{Yevgeny Kazakov}, \bibinfo{person}{Markus Kr{\"o}tzsch}, {and} \bibinfo{person}{Franti{\v s}ek Siman{\v c}{\'i}k}.} \bibinfo{year}{2014}\natexlab{}.
\newblock \showarticletitle{The {{Incredible ELK}}}.
\newblock \bibinfo{journal}{\emph{Journal of Automated Reasoning}} \bibinfo{volume}{53}, \bibinfo{number}{1} (\bibinfo{year}{2014}), \bibinfo{pages}{1--61}.
\newblock


\bibitem[Kingma and Ba(2015)]%
        {kingma2015adam}
\bibfield{author}{\bibinfo{person}{Diederik~P. Kingma} {and} \bibinfo{person}{Jimmy Ba}.} \bibinfo{year}{2015}\natexlab{}.
\newblock \showarticletitle{Adam: {{A}} Method for Stochastic Optimization}. In \bibinfo{booktitle}{\emph{3rd International Conference on Learning Representations}}.
\newblock


\bibitem[Kr{\"o}tzsch(2012)]%
        {krotzsch2012OWL}
\bibfield{author}{\bibinfo{person}{Markus Kr{\"o}tzsch}.} \bibinfo{year}{2012}\natexlab{}.
\newblock \showarticletitle{{{OWL}} 2 {{Profiles}}: {{An Introduction}} to {{Lightweight Ontology Languages}}}.
\newblock In \bibinfo{booktitle}{\emph{Reasoning {{Web}}. {{Semantic Technologies}} for {{Advanced Query Answering}}. {{Reasoning Web}} 2012}}. Vol.~\bibinfo{volume}{7487}. \bibinfo{pages}{112--183}.
\newblock


\bibitem[Kulmanov et~al\mbox{.}(2019)]%
        {kulmanov2019embeddings}
\bibfield{author}{\bibinfo{person}{Maxat Kulmanov}, \bibinfo{person}{Wang {Liu-Wei}}, \bibinfo{person}{Yuan Yan}, {and} \bibinfo{person}{Robert Hoehndorf}.} \bibinfo{year}{2019}\natexlab{}.
\newblock \showarticletitle{{{EL}} Embeddings: {{Geometric}} Construction of Models for the Description Logic {{EL}}++}. In \bibinfo{booktitle}{\emph{Proceedings of the Twenty-Eighth International Joint Conference on Artificial Intelligence}}. \bibinfo{pages}{6103--6109}.
\newblock


\bibitem[Lin et~al\mbox{.}(2015)]%
        {lin2015learning}
\bibfield{author}{\bibinfo{person}{Yankai Lin}, \bibinfo{person}{Zhiyuan Liu}, \bibinfo{person}{Maosong Sun}, \bibinfo{person}{Yang Liu}, {and} \bibinfo{person}{Xuan Zhu}.} \bibinfo{year}{2015}\natexlab{}.
\newblock \showarticletitle{Learning Entity and Relation Embeddings for Knowledge Graph Completion}. In \bibinfo{booktitle}{\emph{Proceedings of the {{Twenty-Ninth AAAI Conference}} on {{Artificial Intelligence}}}}. \bibinfo{pages}{2181--2187}.
\newblock
\showISBNx{978-0-262-51129-2}


\bibitem[Mendez(2012)]%
        {mendez2012jcel}
\bibfield{author}{\bibinfo{person}{Julian Mendez}.} \bibinfo{year}{2012}\natexlab{}.
\newblock \showarticletitle{Jcel: {{A}} Modular Rule-Based Reasoner}. In \bibinfo{booktitle}{\emph{Proceedings of the 1st International Workshop on {{OWL}} Reasoner Evaluation}}, Vol.~\bibinfo{volume}{858}.
\newblock


\bibitem[Mondal et~al\mbox{.}(2021)]%
        {mondala2021emel++}
\bibfield{author}{\bibinfo{person}{Sutapa Mondal}, \bibinfo{person}{Sumit Bhatia}, {and} \bibinfo{person}{Raghava Mutharaju}.} \bibinfo{year}{2021}\natexlab{}.
\newblock \showarticletitle{{{EmEL}}++: {{Embeddings}} for \el\ {{Description Logic}}}. In \bibinfo{booktitle}{\emph{Proceedings of the {{AAAI}} 2021 {{Spring Symposium}} on {{Combining Machine Learning}} and {{Knowledge Engineering}}}} \emph{(\bibinfo{series}{{{CEUR Workshop Proceedings}}}, Vol.~\bibinfo{volume}{2846})}.
\newblock


\bibitem[Mungall et~al\mbox{.}(2012)]%
        {mungall2012uberon}
\bibfield{author}{\bibinfo{person}{Christopher~J. Mungall}, \bibinfo{person}{Carlo Torniai}, \bibinfo{person}{Georgios~V. Gkoutos}, \bibinfo{person}{Suzanna~E. Lewis}, {and} \bibinfo{person}{Melissa~A. Haendel}.} \bibinfo{year}{2012}\natexlab{}.
\newblock \showarticletitle{Uberon, an Integrative Multi-Species Anatomy Ontology}.
\newblock \bibinfo{journal}{\emph{Genome Biology}} \bibinfo{volume}{13}, \bibinfo{number}{1} (\bibinfo{year}{2012}), \bibinfo{pages}{R5}.
\newblock


\bibitem[Nayyeri et~al\mbox{.}(2021)]%
        {nayyeri2021LogicENN}
\bibfield{author}{\bibinfo{person}{Mojtaba Nayyeri}, \bibinfo{person}{Chengjin Xu}, \bibinfo{person}{Mirza~Mohtashim Alam}, \bibinfo{person}{Jens Lehmann}, {and} \bibinfo{person}{Hamed Shariat~Yazdi}.} \bibinfo{year}{2021}\natexlab{}.
\newblock \showarticletitle{{{LogicENN}}: {{A Neural Based Knowledge Graphs Embedding Model}} with {{Logical Rules}}}.
\newblock \bibinfo{journal}{\emph{IEEE Transactions on Pattern Analysis and Machine Intelligence}} (\bibinfo{year}{2021}).
\newblock
\showISSN{1939-3539}


\bibitem[Nayyeri et~al\mbox{.}(2020)]%
        {nayyeri2020Fantastic}
\bibfield{author}{\bibinfo{person}{Mojtaba Nayyeri}, \bibinfo{person}{Chengjin Xu}, \bibinfo{person}{Sahar Vahdati}, \bibinfo{person}{Nadezhda Vassilyeva}, \bibinfo{person}{Emanuel Sallinger}, \bibinfo{person}{Hamed~Shariat Yazdi}, {and} \bibinfo{person}{Jens Lehmann}.} \bibinfo{year}{2020}\natexlab{}.
\newblock \showarticletitle{Fantastic {{Knowledge Graph Embeddings}} and {{How}} to {{Find}} the {{Right Space}} for {{Them}}}. In \bibinfo{booktitle}{\emph{The {{Semantic Web}}}}. \bibinfo{pages}{438--455}.
\newblock
\showISBNx{978-3-030-62419-4}


\bibitem[{\"O}z{\c c}ep et~al\mbox{.}(2020)]%
        {ozcep2020Cone}
\bibfield{author}{\bibinfo{person}{{\"O}zg{\"u}r~L{\"u}tf{\"u} {\"O}z{\c c}ep}, \bibinfo{person}{Mena Leemhuis}, {and} \bibinfo{person}{Diedrich Wolter}.} \bibinfo{year}{2020}\natexlab{}.
\newblock \showarticletitle{Cone {{Semantics}} for {{Logics}} with {{Negation}}}. In \bibinfo{booktitle}{\emph{Proceedings of the {{Twenty-Ninth International Joint Conference}} on {{Artificial Intelligence}}}}, Vol.~\bibinfo{volume}{2}. \bibinfo{pages}{1820--1826}.
\newblock


\bibitem[Paszke et~al\mbox{.}(2019)]%
        {paszke2019pyTorch}
\bibfield{author}{\bibinfo{person}{Adam Paszke}, \bibinfo{person}{Sam Gross}, \bibinfo{person}{Francisco Massa}, \bibinfo{person}{Adam Lerer}, \bibinfo{person}{James Bradbury}, \bibinfo{person}{Gregory Chanan}, \bibinfo{person}{Trevor Killeen}, \bibinfo{person}{Zeming Lin}, \bibinfo{person}{Natalia Gimelshein}, \bibinfo{person}{Luca Antiga}, \bibinfo{person}{Alban Desmaison}, \bibinfo{person}{Andreas Kopf}, \bibinfo{person}{Edward Yang}, \bibinfo{person}{Zachary DeVito}, \bibinfo{person}{Martin Raison}, \bibinfo{person}{Alykhan Tejani}, \bibinfo{person}{Sasank Chilamkurthy}, \bibinfo{person}{Benoit Steiner}, \bibinfo{person}{Lu Fang}, \bibinfo{person}{Junjie Bai}, {and} \bibinfo{person}{Soumith Chintala}.} \bibinfo{year}{2019}\natexlab{}.
\newblock \showarticletitle{{{PyTorch}}: {{An}} Imperative Style, High-Performance Deep Learning Library}. In \bibinfo{booktitle}{\emph{Advances in {{Neural Information Processing Systems}}}}, Vol.~\bibinfo{volume}{32}. \bibinfo{pages}{8024--8035}.
\newblock


\bibitem[Peng et~al\mbox{.}(2022)]%
        {peng2022description}
\bibfield{author}{\bibinfo{person}{Xi Peng}, \bibinfo{person}{Zhenwei Tang}, \bibinfo{person}{Maxat Kulmanov}, \bibinfo{person}{Kexin Niu}, {and} \bibinfo{person}{Robert Hoehndorf}.} \bibinfo{year}{2022}\natexlab{}.
\newblock \bibinfo{title}{Description Logic EL++ Embeddings with Intersectional Closure}.
\newblock
\newblock
\showeprint[arxiv]{2202.14018}


\bibitem[Rector et~al\mbox{.}(1996)]%
        {rector1996galen}
\bibfield{author}{\bibinfo{person}{Alan~L Rector}, \bibinfo{person}{Jeremy~E Rogers}, {and} \bibinfo{person}{Pam Pole}.} \bibinfo{year}{1996}\natexlab{}.
\newblock \showarticletitle{The {GALEN} high level ontology}.
\newblock In \bibinfo{booktitle}{\emph{Medical Informatics Europe '96}}. \bibinfo{publisher}{IOS Press}, \bibinfo{pages}{174--178}.
\newblock


\bibitem[Rockt{\"a}schel et~al\mbox{.}(2015)]%
        {rocktaschel2015Injecting}
\bibfield{author}{\bibinfo{person}{Tim Rockt{\"a}schel}, \bibinfo{person}{Sameer Singh}, {and} \bibinfo{person}{Sebastian Riedel}.} \bibinfo{year}{2015}\natexlab{}.
\newblock \showarticletitle{Injecting {{Logical Background Knowledge}} into {{Embeddings}} for {{Relation Extraction}}}. In \bibinfo{booktitle}{\emph{Proceedings of the 2015 {{Conference}} of the {{North American Chapter}} of the {{Association}} for {{Computational Linguistics}}: {{Human Language Technologies}}}}. \bibinfo{pages}{1119--1129}.
\newblock


\bibitem[Schulz et~al\mbox{.}(2009)]%
        {schulz2009SNOMED}
\bibfield{author}{\bibinfo{person}{Stefan Schulz}, \bibinfo{person}{Boontawee Suntisrivaraporn}, \bibinfo{person}{Franz Baader}, {and} \bibinfo{person}{Martin Boeker}.} \bibinfo{year}{2009}\natexlab{}.
\newblock \showarticletitle{{{SNOMED}} Reaching Its Adolescence: {{Ontologists}}' and Logicians' Health Check}.
\newblock \bibinfo{journal}{\emph{International Journal of Medical Informatics}}  \bibinfo{volume}{78} (\bibinfo{date}{April} \bibinfo{year}{2009}), \bibinfo{pages}{S86--S94}.
\newblock
\showISSN{1386-5056}


\bibitem[Smaili et~al\mbox{.}(2019)]%
        {smaili2019opa2vec}
\bibfield{author}{\bibinfo{person}{Fatima~Zohra Smaili}, \bibinfo{person}{Xin Gao}, {and} \bibinfo{person}{Robert Hoehndorf}.} \bibinfo{year}{2019}\natexlab{}.
\newblock \showarticletitle{{OPA2Vec}: Combining formal and informal content of biomedical ontologies to improve similarity-based prediction}.
\newblock \bibinfo{journal}{\emph{Bioinformatics}} \bibinfo{volume}{35}, \bibinfo{number}{12} (\bibinfo{year}{2019}), \bibinfo{pages}{2133--2140}.
\newblock


\bibitem[Smith et~al\mbox{.}(2007)]%
        {smith2007obo}
\bibfield{author}{\bibinfo{person}{Barry Smith}, \bibinfo{person}{Michael Ashburner}, \bibinfo{person}{Cornelius Rosse}, \bibinfo{person}{Jonathan Bard}, \bibinfo{person}{William Bug}, \bibinfo{person}{Werner Ceusters}, \bibinfo{person}{Louis~J Goldberg}, \bibinfo{person}{Karen Eilbeck}, \bibinfo{person}{Amelia Ireland}, \bibinfo{person}{Christopher~J Mungall}, {et~al\mbox{.}}} \bibinfo{year}{2007}\natexlab{}.
\newblock \showarticletitle{The OBO Foundry: coordinated evolution of ontologies to support biomedical data integration}.
\newblock \bibinfo{journal}{\emph{Nature biotechnology}} \bibinfo{volume}{25}, \bibinfo{number}{11} (\bibinfo{year}{2007}), \bibinfo{pages}{1251--1255}.
\newblock


\bibitem[Staab and Studer(2010)]%
        {staab2010handbook}
\bibfield{author}{\bibinfo{person}{Steffen Staab} {and} \bibinfo{person}{Rudi Studer}.} \bibinfo{year}{2010}\natexlab{}.
\newblock \bibinfo{booktitle}{\emph{Handbook on ontologies}}.
\newblock \bibinfo{publisher}{Springer Science \& Business Media}.
\newblock


\bibitem[Szklarczyk et~al\mbox{.}(2021)]%
        {szklarczyk2021string}
\bibfield{author}{\bibinfo{person}{Damian Szklarczyk}, \bibinfo{person}{Annika~L Gable}, \bibinfo{person}{Katerina~C Nastou}, \bibinfo{person}{David Lyon}, \bibinfo{person}{Rebecca Kirsch}, \bibinfo{person}{Sampo Pyysalo}, \bibinfo{person}{Nadezhda~T Doncheva}, \bibinfo{person}{Marc Legeay}, \bibinfo{person}{Tao Fang}, \bibinfo{person}{Peer Bork}, \bibinfo{person}{Lars~J Jensen}, {and} \bibinfo{person}{Christian {von~Mering}}.} \bibinfo{year}{2021}\natexlab{}.
\newblock \showarticletitle{The {{STRING}} Database in 2021: Customizable Protein\textendash Protein Networks, and Functional Characterization of User-Uploaded Gene/Measurement Sets}.
\newblock \bibinfo{journal}{\emph{Nucleic Acids Research}} \bibinfo{volume}{49}, \bibinfo{number}{D1} (\bibinfo{year}{2021}), \bibinfo{pages}{D605--D612}.
\newblock


\bibitem[Trouillon et~al\mbox{.}(2016)]%
        {trouillon2016Complex}
\bibfield{author}{\bibinfo{person}{Th{\'e}o Trouillon}, \bibinfo{person}{Johannes Welbl}, \bibinfo{person}{Sebastian Riedel}, \bibinfo{person}{Eric Gaussier}, {and} \bibinfo{person}{Guillaume Bouchard}.} \bibinfo{year}{2016}\natexlab{}.
\newblock \showarticletitle{Complex {{Embeddings}} for {{Simple Link Prediction}}}. In \bibinfo{booktitle}{\emph{Proceedings of {{The}} 33rd {{International Conference}} on {{Machine Learning}}}}. \bibinfo{publisher}{{PMLR}}, \bibinfo{pages}{2071--2080}.
\newblock
\showISSN{1938-7228}


\bibitem[Wang et~al\mbox{.}(2017)]%
        {wang2017knowledge}
\bibfield{author}{\bibinfo{person}{Quan Wang}, \bibinfo{person}{Zhendong Mao}, \bibinfo{person}{Bin Wang}, {and} \bibinfo{person}{Li Guo}.} \bibinfo{year}{2017}\natexlab{}.
\newblock \showarticletitle{Knowledge graph embedding: A survey of approaches and applications}.
\newblock \bibinfo{journal}{\emph{IEEE Transactions on Knowledge and Data Engineering}} \bibinfo{volume}{29}, \bibinfo{number}{12} (\bibinfo{year}{2017}), \bibinfo{pages}{2724--2743}.
\newblock


\bibitem[Wang et~al\mbox{.}(2015)]%
        {wang2015Knowledge}
\bibfield{author}{\bibinfo{person}{Quan Wang}, \bibinfo{person}{Bin Wang}, {and} \bibinfo{person}{Li Guo}.} \bibinfo{year}{2015}\natexlab{}.
\newblock \showarticletitle{Knowledge {{Base Completion Using Embeddings}} and {{Rules}}}. In \bibinfo{booktitle}{\emph{Proceedings of the {{Twenty-Fourth International Joint Conference}} on {{Artificial Intelligence}}}}.
\newblock


\bibitem[Wang et~al\mbox{.}(2014)]%
        {wang2014knowledge}
\bibfield{author}{\bibinfo{person}{Zhen Wang}, \bibinfo{person}{Jianwen Zhang}, \bibinfo{person}{Jianlin Feng}, {and} \bibinfo{person}{Zheng Chen}.} \bibinfo{year}{2014}\natexlab{}.
\newblock \showarticletitle{Knowledge graph embedding by translating on hyperplanes}. In \bibinfo{booktitle}{\emph{Proceedings of the AAAI conference on artificial intelligence}}, Vol.~\bibinfo{volume}{28}.
\newblock


\bibitem[Xiang et~al\mbox{.}(2021)]%
        {xiang2021ontoea}
\bibfield{author}{\bibinfo{person}{Yuejia Xiang}, \bibinfo{person}{Ziheng Zhang}, \bibinfo{person}{Jiaoyan Chen}, \bibinfo{person}{Xi Chen}, \bibinfo{person}{Zhenxi Lin}, {and} \bibinfo{person}{Yefeng Zheng}.} \bibinfo{year}{2021}\natexlab{}.
\newblock \showarticletitle{OntoEA: Ontology-guided Entity Alignment via Joint Knowledge Graph Embedding}. In \bibinfo{booktitle}{\emph{Findings of the Association for Computational Linguistics: ACL-IJCNLP 2021}}. \bibinfo{pages}{1117--1128}.
\newblock


\bibitem[Xiong et~al\mbox{.}(2022)]%
        {xiong2022box}
\bibfield{author}{\bibinfo{person}{Bo Xiong}, \bibinfo{person}{Nico Potyka}, \bibinfo{person}{Trung{-}Kien Tran}, \bibinfo{person}{Mojtaba Nayyeri}, {and} \bibinfo{person}{Steffen Staab}.} \bibinfo{year}{2022}\natexlab{}.
\newblock \showarticletitle{Faithful Embeddings for \el\ Knowledge Bases}. In \bibinfo{booktitle}{\emph{Proceedings of the 21st International Semantic Web Conference}}, Vol.~\bibinfo{volume}{13489}. \bibinfo{pages}{22--38}.
\newblock


\bibitem[Yang et~al\mbox{.}(2015)]%
        {yang2015embedding}
\bibfield{author}{\bibinfo{person}{Bishan Yang}, \bibinfo{person}{Scott Wen-tau Yih}, \bibinfo{person}{Xiaodong He}, \bibinfo{person}{Jianfeng Gao}, {and} \bibinfo{person}{Li Deng}.} \bibinfo{year}{2015}\natexlab{}.
\newblock \showarticletitle{Embedding Entities and Relations for Learning and Inference in Knowledge Bases}. In \bibinfo{booktitle}{\emph{Proceedings of the International Conference on Learning Representations}}.
\newblock


\bibitem[Zhai et~al\mbox{.}(2010)]%
        {zhai2010geospatial}
\bibfield{author}{\bibinfo{person}{Xiaofang Zhai}, \bibinfo{person}{Lei Huang}, {and} \bibinfo{person}{Zhifeng Xiao}.} \bibinfo{year}{2010}\natexlab{}.
\newblock \showarticletitle{Geo-spatial query based on extended SPARQL}. In \bibinfo{booktitle}{\emph{2010 18th International Conference on Geoinformatics}}. \bibinfo{pages}{1--4}.
\newblock


\end{thebibliography}
\clearpage

\appendix

\section{Semantics of \texorpdfstring{\el}{EL++}}\label{sec:semantics}
\newcommand{\I}{\mathcal{I}}
The semantics of \el\ are defined in terms of interpretations. An \textit{interpretation} is a tuple $\mathcal I = (\Delta^{\mathcal I}, \cdot^{\mathcal I})$ of a non-empty set $\Delta^{\I}$, the \textit{interpretation domain}, and a mapping function $\cdot^{\I} \colon N_C\cup N_I\cup N_R\to\Delta^{\I}$, that maps
    \begin{itemize}
        \item individuals $a\in\mathcal N_I$ to objects $a^{\I}\in\Delta^{\I}$;
        \item concept names $C\in\mathcal N_C$ to subsets $C^{\I}\subseteq \Delta^{\I}$; and
        \item role names $r\in \mathcal N_R$ to binary relations $r^{\I}\subseteq \Delta^{\I}\times \Delta^{\I}$.
    \end{itemize}

The mapping $\cdot^{\I}$ is extended to arbitrary concepts as follows:
\begin{align*}
    \top^{\I} &= \Delta^{\I},\\
    \bot^{\I} &= \emptyset,\\
    \{a\}^{\I} &= \{a^{\I}\},\\
    (C\sqcap D)^{\I} &= C^{\I} \cap D^{\I},\\
    (\exists r.C)^{\I} &= \{\, x\in\Delta^{\I} \mid \exists y\in\Delta^{\I}.\, (x, y)\in r^{\I} \land y\in C^{\I} \,\}.
\end{align*}
An interpretation thus establishes a precise correspondence between syntactical \el\ concepts and sets of objects in the interpretation domain. An interpretation $\I$ \textit{satisfies} a concept subsumption axiom $C\sqsubseteq D$ if $C^\I \subseteq D^\I$ and a role inclusion axiom $r_1\circ\dots\circ r_k\sqsubseteq r$ if $r_1^\I;\,\dots;\,r_k^\I\sqsubseteq r^\I$, where $;$ denotes relation composition. Furthermore, $\I$ satisfies a concept assertion $C(a)$ if $a^\I \in C^\I$ and a role assertion $r(a,b)$ if $(a^\I,b^\I)\in r^\I$.
An interpretation that satisfies each axiom and assertion in an ontology $\ont$ is called a \textit{model} of $\ont$, denoted as $\I\models\ont$.

\section{Normalization Procedure}\label{sec:normalization}
We give a brief overview of the normalization procedure for \el\ ontologies introduced in \cite{baader2005pushing}. 
Given an \el\ TBox $\mathcal T$, let $\bc = \mathcal N_C \cup \{\top\} \cup \{\{a\} \mid a \in \mathcal N_I\}$ be the set of \textit{basic concept descriptions}. $\mathcal T$ is in \textit{normal form} if (1) all concept subsumption axioms are of one of the forms
\begin{align*}
    C_1&\sqsubseteq D\  &
    C_1\sqcap C_2&\sqsubseteq D\\
    C_1&\sqsubseteq\exists r. C_2 &
    \exists r.C_1&\sqsubseteq D,
\end{align*}
where $C_1, C_2\in \bc$ and $D\in\bc\cup\{\bot\}$, and (2) all role inclusion axioms are of the form $r\sqsubseteq s$ or $r_1\circ r_2\sqsubseteq s$. Any \el\ TBox $\mathcal T$ can be converted into a normalized TBox $\mathcal T'$ whose size is linear in the size of $\mathcal T$ by exhaustively applying a number of translation rules that introduce new concept and role names. The details of these translation rules are given in \cite{baader2005Pushinga}. The resulting ontology is a \textit{conservative extension} of $\mathcal T$; that is, every model of $\mathcal T'$ is also a model of $\mathcal T$. Since the ABox of an ontology can be readily converted into equivalent TBox axioms (see \cref{sec:training}), this procedure equally applies to normalizing entire ontologies.

\section{Proof of Theorem 1 (Soundness)}\label{sec:soundness_proof}
\newcommand{\itheta}{\mathcal{I}}
In order to prove \cref{theo:soundness}, we first show the correctness of our loss functions. Recall that the inclusion loss $\loss_{\subseteq}(A, B)$ of two boxes $A$ and $B$ is defined as
\begin{equation*}
        \mathcal L_\subseteq(A, B) = 
    \begin{cases}
        \norm{\max\{\bm 0,\, \bm d(A, B) + 2\bm o(A) - \gamma\}} & \text{if } B \neq \emptyset\\
        \max\{0, \bm o(A)_1 + 1\} & \text{otherwise}.
    \end{cases}
\end{equation*}

\begin{lemma}
    \label{lemma:inclusion}
    Let $A$ and $B$ be boxes in $\mathbb R^n$. If $\gamma\leq 0$ and $\loss_{\subseteq}(A, B) = 0$, then $A\subseteq B$.
\end{lemma}
\newcommand{\ua}{\bm u_{A,\,k}}
\newcommand{\ub}{\bm u_{B,\,k}}
\newcommand{\la}{\bm l_{A,\,k}}
\newcommand{\lb}{\bm l_{B,\,k}}

\begin{proof}
    First, assume that $B = \emptyset$. Since $\loss_\subseteq(A,B) = 0$, we have $\bm o(A)_1 = \bm u_{A,1} - \bm l_{A,1} \leq -1$ and thus $\bm l_{A,1} > \bm u_{A,1}$. Therefore $A = \emptyset$.

    If $B\neq\emptyset$, we show that $\bm l_B\leq \bm l_A$ and $\bm u_A\leq \bm u_B$. Assume $\loss_{\subseteq}(A, B) = 0$. We have that
    \begin{align*}
        \bm d(A, B) + 2\bm o(A) - \gamma &\leq 0\\
        \abs{\bm c(A) - \bm c(B)} + \bm o(A) - \bm o(B) - \gamma &\leq 0
    \end{align*}
    and thus
    \begin{equation*}
        \abs{\bm c(A) - \bm c(B)} + \bm o(A) - \bm o(B) \leq \gamma \leq 0.
    \end{equation*}
    Now fix an arbitrary dimension $k$ such that $1\leq k\leq n$. We distinguish two cases:

    \textbf{Case 1: $\bm c(A)_k \geq \bm c(B)_k.$}
    We eliminate the absolute value function and use $\bm u_\beta = \bm c(\beta) + \bm o(\beta)$ for an arbitrary box $\beta$ to obtain
    \begin{align*}
        \ua - \ub &\leq 0\\
        \ua &\leq \ub.
    \end{align*}
    Since $\bm c(A)_k \geq \bm c(B)_k$, we furthermore have by the definition of $\bm c(\cdot)$ that
    \begin{align*}
        \frac{\la+\ua}{2}&\geq \frac{\lb+\ub}{2}\\[-.1cm]
        \la&\geq \lb+\underbrace{\ub-\ua}_{\geq 0}\\[-.45cm]
        \la&\geq \lb.
    \end{align*}
        
    \textbf{Case 2: $\bm c(A)_k \leq \bm c(B)_k.$}
    Similarly to the first case, we eliminate the absolute value function and use $\bm l_\beta = \bm c(\beta) - \bm o(\beta)$ to obtain
    \begin{align*}
        -\la + \lb &\leq 0\\
        \lb &\leq \la.
    \end{align*}
    Because $\bm c(A)_k \leq \bm c(B)_k$, we furthermore have
    \begin{align*}
        \frac{\la+\ua}{2}&\leq \frac{\lb+\ub}{2}\\[-.1cm]
        \underbrace{\la - \lb}_{\geq 0} +\,\ua &\leq \ub\\[-.45cm]
        \ua &\leq \ub.
    \end{align*}

    Now, consider an arbitrary point $\bm a\in A$, i.e., $\bm l_A\leq \bm a\leq \bm u_A$. But then
    \begin{equation*}
        \bm l_B\leq \bm l_A\leq \bm a \leq \bm u_A\leq \bm u_B
    \end{equation*}
    and thus $\bm a\in B$.
\end{proof}

\begin{lemma}
    \label{lemma:disjoint}
    Let $A$ and $B$ be boxes in $\mathbb R^n$. If $\gamma\leq 0$ and
    \begin{equation*}
    \norm{\max\{\bm 0,\, -(\bm d(A, B) + \gamma)\}} = 0,
    \end{equation*}
    then $A\cap B = \emptyset$.
\end{lemma}

\begin{table*}
    \centering
    \caption{Sizes of the different ontologies we consider. The number of classes, roles, and axioms in each normal form is reported.}
    \label{tab:datasets}
    \begin{tabular}{@{}lrrrrrrrrr@{}}
        \toprule
        Ontology & Classes & Roles & $C\sqsubseteq D$  & $C\sqcap D\sqsubseteq E$    & $C\sqsubseteq \exists r. D$     & $\exists r. C\sqsubseteq D$    & $C\sqcap D\sqsubseteq \bot$ & $r\sqsubseteq s$ & $r_1 \circ r_2 \sqsubseteq s$\\ \midrule
        GALEN    & 24,353  & 951   & 28,890  & 13,595 & 28,118  & 13,597 & 0 & 958 & 58 \\
        GO       & 45,907  & 9     & 85,480  & 12,131 & 20,324  & 12,129 & 30 & 3 & 6\\
        Anatomy  & 106,495 & 188   & 122,142 & 2,121  & 152,289 & 2,143  & 184 & 89 & 31\\ \bottomrule
    \end{tabular}
\end{table*}

\begin{proof}
    The proof is similar to that of \cref{lemma:inclusion}. We have that
    \begin{align*}
        -(\bm d(A,B)+\gamma)&\leq 0\\
        -(\abs{\bm c(A) - \bm c(B)} - \bm o(A) - \bm o(B) + \gamma) &\leq 0
    \end{align*}
    and therefore
    \begin{equation*}
        \abs{\bm c(A) - \bm c(B)} - \bm o(A) - \bm o(B)\geq -\gamma \geq 0.
    \end{equation*}
    We again fix a dimension $k$ such that $1\leq k\leq n$ and distinguish two cases:
    
    \textbf{Case 1: $\bm c(A)_k \geq \bm c(B)_k.$}
    Eliminating the absolute value function yields
    \begin{align*}
        \la - \ub\ &\geq 0 \nonumber\\
        \la &\geq \ub.
    \end{align*}
    
    \textbf{Case 2: $\bm c(A)_k \leq \bm c(B)_k.$}
    Analogously to Case 1, we have
    \begin{align*}
        \lb - \ua &\geq 0 \nonumber\\
        \lb &\geq \ua.
    \end{align*}
    Now consider an arbitrary point $\bm a\in A$. From the case analysis above, we know that either $\bm a_k \geq \ub$ or $\bm a_k \leq \lb$. However, in both cases $\bm a$ cannot be in $B$.
\end{proof}

We are now ready to prove \cref{theo:soundness}, restated below.

\begingroup
\def\thetheorem{1}
\begin{theorem}[Soundness]
    Let $\mathcal O = (\mathcal T, \mathcal A)$ be an \el\ ontology. If $\gamma\leq 0$ and there exist \boxsqel\ embeddings in $\mathbb R^n$ such that $\loss(\mathcal O) = 0$, then these embeddings are a model of $\mathcal O$.
\end{theorem}
\addtocounter{theorem}{-1}
\endgroup

\begin{proof}
We first perform the standard steps of transforming the ABox and normalizing the axioms in $\mathcal O$. Let $\mathcal O'$ denote the resulting ontology. The \boxsqel\ embeddings induce the geometric interpretation $\mathcal I = (\Delta^{\itheta}, \cdot^{\itheta})$ defined as follows:
\begin{enumerate}
    \item $\Delta^{\itheta} = \mathbb R^n$,
    \item for every concept $C\in \mathcal N_C$, let $C^{\itheta} = \boxx(C)$,
    \item for every individual $a\in \mathcal N_I$, let $a^{\itheta} = \bm e_a$,
    \item for every role $r\in \mathcal N_R$, let $r^{\itheta} = \head(r) \times \tail(r)$.
\end{enumerate}

We show that $\mathcal I$ is a model of $\ont'$. First, note that $\loss(\mathcal O) = 0$ implies that the regularization loss is 0, and thus $\bump(C) = \bm 0$ for any $C\in \mathcal N_C \cup \mathcal N_I$. We now show that $\itheta$ satisfies every axiom $\alpha\in\mathcal O'$, distinguishing between the different normal forms. Implicitly, we make frequent use of \cref{lemma:inclusion}, which we do not state explicitly for the sake of brevity.

\textbf{Case 1:\mbox{ $\alpha = C\sqsubseteq D.$}}
Since $\loss_1(C, D) = 0$ and therefore\linebreak $\loss_{\subseteq}(\boxx(C), \boxx(D)) = 0$, we have that $\boxx(C)\subseteq\boxx(D)$. But then it immediately follows from the definition of $\itheta$ that $C^{\itheta}\subseteq D^{\itheta}$.

\textbf{Case 2: \mbox{$\alpha = C\sqcap D\sqsubseteq E.$}}
We have that $\loss_2(C, D, E) = 0$ and therefore it follows that $\boxx(C)\cap\boxx(D)\subseteq \boxx(E)$. Hence, we have $(C\sqcap D)^{\itheta} = C^{\itheta}\cap D^{\itheta} = \boxx(C)\cap\boxx(D)\subseteq\boxx(E)=E^{\itheta}$.

\textbf{Case 3: \mbox{$\alpha = C\sqsubseteq \exists r. D.$}}
Assume $D^{\mathcal I}\neq\emptyset$.
Let $x\in C^{\itheta} = \boxx(C)$. Since $\loss_3(C, r, D) = 0$ and all bump vectors are $\bm 0$, we have $\boxx(C) \subseteq \head(r)$ and therefore $x\in\head(r)$. Similarly, for any $y\in D^{\itheta}$ we have $y\in\tail(r)$. But then $(x, y)\in r^{\itheta}$ and therefore $x\in (\exists r. D)^{\itheta}$.

If on the other hand $D^{\mathcal I} = \emptyset$, we also have $(\exists r. D)^{\mathcal I} = \emptyset$. Since $\boxx(D)$ is empty, we furthermore have $\loss_\subseteq(\boxx(C), \emptyset) = 0$ and therefore $\boxx(C) = \emptyset$, i.e., $C^{\itheta}\subseteq (\exists r. D)^{\mathcal I}$.

\textbf{Case 4: \mbox{$\alpha = \exists r. C \sqsubseteq D.$}}
Assume $D^{\mathcal I}\neq\emptyset$.
Let $x\in (\exists r. C)^{\itheta}$. Hence, there exist a $y\in C^{\itheta}$ such that $(x,y)\in r^{\itheta}$. By the definition of $r^{\itheta}$, we must therefore have $x\in \head(r)$. Since $\loss_4(r, C, D) = 0$, furthermore $\head(r)\subseteq\boxx(D)$ and therefore $x\in D^{\itheta}$.

If on the other hand $D^{\mathcal I} = \emptyset$, we have $\head(r)\subseteq\emptyset$ and thus $r=\emptyset$, so $(\exists r. C)^{\itheta} = \emptyset$.

\textbf{Case 5: \mbox{$\alpha = C\sqcap D\sqsubseteq\bot.$}}
We have $\loss_5(C, D) = 0$, so by \cref{lemma:disjoint} we have that $(C\sqcap D)^{\itheta} = \boxx(C)\cap \boxx(D) = \emptyset \subseteq \bot^{\itheta}$.

\textbf{Case 6: \mbox{$\alpha = r\sqsubseteq s.$}} Let $(a,b)\in r^{\itheta}$. By the definition of $r^{\itheta}$, $a\in\head(r)$ and $b\in\tail(r)$. Since $\loss_6(r, s) = 0$, we furthermore have $\head(r)\subseteq\head(s)$ and $\tail(r)\subseteq\tail(s)$ and hence $(a,b)\in s^{\itheta}$.

\textbf{Case 7: \mbox{$\alpha = r_1\circ r_2\sqsubseteq s.$}} Let $(a,b)\in r_1^\itheta \circ r_2^\itheta$. By definition, we have $a\in\head(r_1)$ and $b\in\tail(r_2)$. Because $\loss_7(r_1,r_2,s) = 0$, $\head(r_1)\subseteq \head(s)$ and $\tail(r_2)\subseteq\tail(s)$, so $(a,b)\in s^\itheta$.

\vspace{.2cm}
We have shown that $\itheta$ satisfies every axiom in $\mathcal O'$, and is therefore a model of $\mathcal O'$. But since $\mathcal O'$ is a conservative extension of $\mathcal O$~\citep{baader2005pushing}, it follows that $\itheta$ is also a model of $\mathcal O$.
\end{proof}

\section{Statistical information about benchmark ontologies}

The sizes of the ontologies we consider in terms of number of classes, roles, and axioms are reported in \cref{tab:datasets}.

\section{Ablation studies}
\label{sec:ablation}
We conduct two ablation studies to investigate the performance impact of different parts of our model. All studies are conducted on the GALEN ontology for the subsumption prediction task and we report results combined across all normal forms.

\subsection{Impact of Role Representation}
We consider an alternative model that uses the exact same optimization procedure as \boxsqel, but represents roles as translations, similar to previous methods. The results are listed in \cref{tab:impact_role}.

We observe that the model that represents roles as translations performs worse on all metrics, in most cases by a large margin. The results highlight the importance of the novel role representation for the performance of our model.

\begin{table}[t]
    \centering
    \caption{Impact of the role representation on the performance of \boxsqel\@. We compare our model with a version in which roles are represented as translations (\boxsqel-Tr).}
    \label{tab:impact_role}
    \resizebox{\columnwidth}{!}{%
        \begin{tabular}{@{}lrrrrrrr@{}}
            \toprule
                           Model       & H@1        & H@10       & H@100      & Med           & MRR           & MR            & AUC           \\ \midrule
\boxsqel-Tr & 0.03 & 0.15 & 0.30 & 1141 & 0.07 & 4793 & 0.79          \\
            \boxsqel    & \textbf{0.05} & \textbf{0.20} & \textbf{0.35} & \textbf{669} & \textbf{0.10} & \textbf{4375} & \textbf{0.81} \\ \bottomrule
        \end{tabular}%
    }
\end{table}

\subsection{Number of Negative Samples}
Our second ablation study concerns the performance impact of the number of negative samples. We report results for \boxsqel\ models trained with 0-5 negative samples per NF3 axiom in \cref{tab:negative_samples}.

We observe that the model that uses no negative samples performs significantly worse than the other models, demonstrating that negative sampling is essential to learn strong embeddings. Using more than one negative sample per NF3 axiom further improves the results, but only marginally.

\begin{table}
    \centering
    \caption{Impact of the number of negative samples on the performance of \boxsqel\@. The model \boxsqel-$\omega$ denotes \boxsqel\ trained with $\omega$ negative samples per NF3 axiom.}
    \label{tab:negative_samples}
    \resizebox{\columnwidth}{!}{
    \begin{tabular}{@{}lrrrrrrr@{}}
        \toprule
        Model      & H@1        & H@10 & H@100      & Med          & MRR           & MR   & AUC           \\ \midrule
        \boxsqel-0 & 0.00 & 0.01 & 0.06 & 7351 & 0.01 & 8727 & 0.62\\
        \boxsqel-1 & \textbf{0.05} & \textbf{0.20} & \textbf{0.35} & 676 & \textbf{0.10} & 4397 & 0.81 \\
        \boxsqel-2 & \textbf{0.05} & \textbf{0.20} & \textbf{0.35} & 638 & \textbf{0.10} & 4255 & \textbf{0.82} \\
        \boxsqel-3 & \textbf{0.05} & 0.19 & \textbf{0.35} & \textbf{625} & 0.09 & 4187 & \textbf{0.82} \\ 
        \boxsqel-4 & \textbf{0.05} & 0.19 & \textbf{0.35} & 628 & \textbf{0.10} & 4177 & \textbf{0.82}\\
        \boxsqel-5 & \textbf{0.05} & 0.19 & \textbf{0.35} & 627 & \textbf{0.10} & \textbf{4174} & \textbf{0.82}\\
        \bottomrule
    \end{tabular}%
    }
\end{table}

\section{Scoring functions}
\label{sec:detailed_scoring}

Scoring functions are used to compute the likelihood of candidate axioms based on the learned embeddings of their concepts. We define a scoring function $s(\cdot)$ for candidate axioms in all four normal forms NF1--4. 

\paragraph{First and second normal form.}
For an NF1 axiom $C\sqsubseteq D$, the score is based simply on the distance between the embeddings of $C$ and $D$, i.e., for \boxsqel\ we have
\begin{equation*}
    s(C\sqsubseteq D) = -\norm{\bm c(\boxx(C)) - \bm c(\boxx(D))}.
\end{equation*}
The same formulation can be used for the baseline methods. Similarly, for NF2 axioms $C\sqcap D\sqsubseteq E$, the score is defined as the negative distance of the embedding of $E$ to the intersection of $C$ and $D$ in the embedding space.

\paragraph{Third normal form.}
For axioms in the third and fourth normal form the scoring function differs between \boxsqel\ and the baseline methods, because of the different role representation. For \boxsqel, the score for a subsumption $C\sqsubseteq\exists r. D$ is naturally defined as
\begin{align*}
        s(C\sqsubseteq\exists r. D) = &-\norm{\bm c(\boxx(C)) + \bump(D) - \bm c(\head(r))}\\
        &- \norm{\bm c(\boxx(D)) + \bump(C) - \bm c(\tail(r))}.
\end{align*}
In the baseline methods, the score is computed similarly to TransE:
\begin{equation*}
    s(C\sqsubseteq\exists r. D) = -\norm{\bm c(\boxx(C)) + \bm v(r) - \bm c(\boxx(D))},
\end{equation*}
where $\bm v(r)$ is the embedding of role $r$.

\paragraph{Fourth normal form.}
Finally, for an NF4 axiom $\exists r. C\sqsubseteq D$, the score assigned by \boxsqel\ is given by
\begin{equation*}
    s(\exists r. C\sqsubseteq D) = -\norm{\bm c(\head(r)) - \bump(C) - \bm c(\boxx(D))},
\end{equation*}
and for the baseline methods we have
\begin{equation*}
    s(\exists r. C\sqsubseteq D) = -\norm{\bm c(\boxx(C)) - v(r) - \bm c(\boxx(D))}.
\end{equation*}

\paragraph{Volume-based scoring functions.}
An alternative, which is employed by BoxEL~\citep{xiong2022box}, is to define scoring functions based on volumes instead of distances. We empirically find distance-based scoring functions to perform much better, which is why we adopt this approach in \boxsqel\@.

\section{Evaluation metrics}\label{sec:metrics}
Let $\test$ denote the testing set and $\rk(\alpha)$ be the rank of the true test axiom $\alpha$ with respect to all candidate predictions. Hits@$k$ measures the fraction of test axioms with rank $\leq k$, i.e.,
\begin{equation*}
    \frac1{|\test|}\sum_{\alpha\in\test}\mathbb 1[\rk(\alpha)\leq k].
\end{equation*}
The mean rank is simply the average of the ranks of all axioms in the testing set.
Often the median rank is used instead, since this metric is more robust to outliers. An alternative is the mean reciprocal rank, defined as
\begin{equation*}
    \frac1{|\test|}\sum_{\alpha\in\test} \frac1{\rk(\alpha)},
\end{equation*}
which takes on values between 0 and 1 only. Finally, to compute the area under the ROC curve we regard the model as a binary classifier parameterized by a threshold $k$, that assigns \textit{true} to a candidate prediction if its rank is less than or equal to $k$ and \textit{false} otherwise. The ROC curve is then obtained by plotting the true positive rate against the false positive rate for varying thresholds $k$.

\section{Hyperparameters}
\label{sec:hyperparameters}

\begin{table*}[t]
\centering
\caption{Hyperparameters for subsumption prediction (a), deductive reasoning (b), and PPI prediction~(c). }
\label{tab:hyper}
\begin{subtable}[t]{.48\linewidth}
    \centering
    \caption{}
    \begin{tabular}{@{}llcccccc@{}}
\toprule
                                         & Model    & $n$ & $\gamma$ & lr        & $\delta$ & $\omega$ & $\lambda$ \\ \midrule
\multirow{5}{*}{\rotatebox{90}{GALEN}}   & ELEm     & 200 & 0.05     & $5\times 10^{-4}$ & --       & --       & --        \\
                                         & \emelpp  & 200 & 0.05     & $5\times 10^{-4}$ & --       & --       & --        \\
                                         & BoxEL  & 50 & --     & $1\times 10^{-3}$ & --       & --       & --        \\
                                         & ELBE     & 200 & 0.05     & $5\times 10^{-3}$ & --       & --       & --        \\
                                         & \boxsqel & 200 & 0.15     & $1\times 10^{-2}$ & 5        & 1        & 0.4      \\ \midrule
\multirow{5}{*}{\rotatebox{90}{GO}}      & ELEm     & 200 & 0.1      & $5\times 10^{-4}$ & --       & --       & --        \\
                                         & \emelpp  & 200 & 0.1      & $5\times 10^{-4}$ & --       & --       & --        \\
                                         & BoxEL  & 25 & --     & $1\times 10^{-2}$ & --       & --       & --        \\
                                         & ELBE     & 200 & 0.1      & $5\times 10^{-3}$ & --       & --       & --        \\
                                         & \boxsqel & 200 & 0.15     & $1\times 10^{-2}$ & 5.5        & 5        & 0.5      \\ \midrule
\multirow{5}{*}{\rotatebox{90}{Anatomy}} & ELEm     & 200 & 0.05     & $5\times 10^{-4}$ & --       & --       & --        \\
                                         & \emelpp  & 200 & 0.05     & $5\times 10^{-4}$ & --       & --       & --        \\
                                         & BoxEL  & 25 & --     & $1\times 10^{-3}$ & --       & --       & --        \\
                                         & ELBE     & 200 & 0.05     & $5\times 10^{-3}$ & --       & --       & --        \\
                                         & \boxsqel & 200 & 0.05     & $1\times 10^{-2}$ & 5.5        & 3        & 0.3      \\ \bottomrule
\end{tabular}%
\end{subtable}%
\begin{subtable}[t]{.48\linewidth}
    \centering
    \caption{}
    \begin{tabular}{@{}llcccccc@{}}
\toprule
                                         & Model    & $n$ & $\gamma$ & lr        & $\delta$ & $\omega$ & $\lambda$ \\ \midrule
\multirow{5}{*}{\rotatebox{90}{GALEN}}   & ELEm     & 200 & 0        & $5\times 10^{-4}$ & --       & --       & --        \\
                                         & \emelpp  & 200 & 0        & $5\times 10^{-4}$ & --       & --       & --        \\
                                         & BoxEL  & 25 & --        & $5\times 10^{-3}$ & --       & --       & --        \\
                                         & ELBE     & 200 & 0.15      & $1\times 10^{-2}$ & --       & --       & --        \\
                                         & \boxsqel & 200 & 0.05     & $5\times 10^{-3}$ & 1        & 2        & 0         \\ \midrule
\multirow{5}{*}{\rotatebox{90}{GO}}      & ELEm     & 200 & 0.1      & $1\times 10^{-3}$ & --       & --       & --        \\
                                         & \emelpp  & 200 & 0.1      & $1\times 10^{-3}$ & --       & --       & --        \\
                                         & BoxEL  & 25 & --        & $1\times 10^{-3}$ & --       & --       & --        \\
                                         & ELBE     & 200 & 0.05     & $5\times 10^{-3}$ & --       & --       & --        \\
                                         & \boxsqel & 200 & 0.05     & $5\times 10^{-3}$ & 3        & 3        & 0.05      \\ \midrule
\multirow{5}{*}{\rotatebox{90}{Anatomy}} & ELEm     & 200 & 0.05     & $5\times 10^{-4}$ & --       & --       & --        \\
                                         & \emelpp  & 200 & 0.05     & $5\times 10^{-4}$ & --       & --       & --        \\
                                         & BoxEL  & 25 & --        & $1\times 10^{-3}$ & --       & --       & --        \\
                                         & ELBE     & 200 & 0.05     & $5\times 10^{-3}$ & --       & --       & --        \\
                                         & \boxsqel & 200 & 0.05     & $1\times 10^{-3}$ & 2        & 2        & 0.05      \\ \bottomrule
\end{tabular}%
\end{subtable}%
\\
\bigskip
\begin{subtable}[t]{\linewidth}
    \centering
    \caption{}
    \begin{tabular}{@{}lcccccc@{}}
\toprule
      & $n$ & $\gamma$ & lr        & $\delta$ & $\omega$ & $\lambda$ \\ \midrule
Yeast & 200 & 0.02     & $1\times 10^{-2}$ & 2.5        & 4        & 0.2      \\ \midrule
Human & 200 & 0.005     & $5\times 10^{-2}$ & 3.5        & 5        & 0.3      \\ \bottomrule
\end{tabular}
\end{subtable}%
\end{table*}

For all experiments, embeddings are learned with the Adam optimizer~\citep{kingma2015adam} for a maximum of 10,000 epochs. All hyperparameters are selected based on validation set performance. The values considered are: dimensionality $n \in \{25, 50, 100, 200\}$, margin $\gamma \in \{0, 0.005, 0.02, 0.05, 0.1, 0.15\}$, learning rate lr $\in \{5\times 10^{-2}, 1\times 10^{-2}, 5\times 10^{-3}, 1\times 10^{-3}, 5\times 10^{-4}\}$, negative sampling distance $\delta$ in steps of size $0.5$ from 0 to $5.5$, number of negative samples $\omega\in\{1, 2, 3, 4, 5\}$, and regularization $\lambda\in\{0, 0.05, 0.1, 0.2, 0.3, 0.4, 0.5\}$.

We report the hyperparameters used in the experiments in \cref{tab:hyper}. Note that for ELBE on GALEN in subsumption prediction, and for \boxsqel\ on GALEN in deductive reasoning and Yeast in PPI prediction, we furthermore decay the learning rate with a factor of 0.1 after 2000 training epochs. For PPI prediction, we only report the hyperparameters for \boxsqel, since the baseline results are taken from the literature.

\begin{figure*}[t]
    \centering
    \vspace{1cm}
    \begin{minipage}{\textwidth}
    \begin{gather}
        \inferrule
        {\mathsf{SodiumLactate}\sqsubseteq\mathsf{NAMEDComplexChemical}\\
        \label{eq:firstInference}
        \mathsf{NAMEDComplexChemical}\sqsubseteq\mathsf{ChemicalSubstance}}
        {\mathsf{SodiumLactate}\sqsubseteq\mathsf{ChemicalSubstance}} \\[.4cm]
        \label{eq:secondInference}
        \mathsf{SodiumLactate}\sqsubseteq \exists \mathsf{isMadeOf}.\mathsf{Sodium}\\[.4cm]
        \inferrule
        {\eqref{eq:firstInference}\\ \eqref{eq:secondInference}\\
        \mathsf{ChemicalSubstance}\sqcap\exists\mathsf{isMadeOf}.\mathsf{Sodium}\sqsubseteq\mathsf{SodiumCompound}}
        {\mathsf{SodiumLactate}\sqsubseteq\mathsf{SodiumCompound}}
    \end{gather}
    \end{minipage}
    \caption{Inference steps required to derive $\mathsf{SodiumLactate}\sqsubseteq \mathsf{SodiumCompound}$. The derivations are from top to bottom, similar to natural deduction.}
    \label{fig:inferences_chemical}
\end{figure*}

\section{Comparison of deductive and inductive reasoning}
\label{sec:deductive_vs_inductive}
We investigate the comparatively lower performance of the embedding models on deductive reasoning compared to the inductive subsumption prediction task with a concrete example. In the GALEN ontology, we find the following subsumption axiom to predict in the testing set:
\begin{equation*}
    \label[axiom]{ax:prediction}
    \mathsf{SodiumLactate}\sqsubseteq\mathsf{NAMEDComplexChemical}.
\end{equation*}
This axiom cannot be logically inferred from the training data. However, the following similar subsumptions are part of the training data:
\begin{align*}
    \mathsf{SodiumLactate}&\sqsubseteq\exists \textsf{isMadeOf}. \mathsf{Sodium}\\
    \mathsf{SodiumBicarbonate}&\sqsubseteq\exists\textsf{isMadeOf}. \mathsf{Sodium}\\
    \mathsf{SodiumCitrate}&\sqsubseteq\exists \textsf{isMadeOf}. \mathsf{Sodium}\\
    \mathsf{SodiumBicarbonate}&\sqsubseteq\mathsf{NAMEDComplexChemical}\\
    \mathsf{SodiumCitrate}&\sqsubseteq\mathsf{NAMEDComplexChemical}.
\end{align*}
It seems quite likely that our model will be able to exploit the statistical information contained in these subsumptions to learn embeddings where \tsf{SodiumLactate} is close to \tsf{NAMEDComplexChemical}, and thus achieve a high score for the desired true axiom.

In contrast, in the deductive reasoning setting, all testing axioms are logical inferences of the training data. One such testing axiom is
\begin{equation*}
    \mathsf{SodiumLactate}\sqsubseteq \mathsf{SodiumCompound}.
\end{equation*}
This logical inference has been reached by a reasoning algorithm such as ELK~\citep{kazakov2014Incredible} by performing the sequence of derivations listed in \cref{fig:inferences_chemical}.
Embedding methods, on the other hand, do not explicitly perform these derivations. Instead, they learn embeddings that align with the semantics of the ontology, i.e., in convergence we will have that $\boxx(\asf{SodiumLactate}) \subseteq \boxx(\asf{SodiumCompound})$. However, our model has to learn highly accurate embeddings for a number of different concepts and roles involved in the derivation sequence for this to actually hold in practice. Furthermore, the model then makes statistical predictions based on the scoring functions instead of only predicting axioms that completely align with the semantics. Hence there may be many other plausible axioms that are ranked higher by the model, decreasing the score of the desired logical inference.
\end{document}